\PassOptionsToPackage{section}{placeins} 
\documentclass[]{ailab}

\newcommand{\satr}{\textsc{SAT}\xspace}
\newcommand{\satrfull}{\textbf{S}taleness-\textbf{A}daptive \textbf{T}rust Region\xspace}
\AtBeginDocument{%
  \xspaceaddexceptions{-}
  \renewcommand{\eqref}[1]{\textup{(\ref{#1})}}
}

\usepackage{tabularx}
\usepackage{url}
\usepackage{adjustbox}
\usepackage{multirow}
\usepackage{pifont}
\usepackage{comment}
\usepackage{amsmath,amssymb}
\usepackage{amsthm}
\usepackage{mathtools}
\usepackage{colortbl}
\usepackage{color}
\usepackage{booktabs}
\usepackage{natbib}
\setcitestyle{square,comma,numbers,sort&compress}
\usepackage{graphicx}
\usepackage{subcaption}
\usepackage{wrapfig}
\usepackage{tcolorbox}
\usepackage[dvipsnames]{xcolor}
\RequirePackage{xspace}
\usepackage{nicefrac}
\usepackage{enumitem}
\usepackage{float}
\usepackage[section]{placeins}
\usepackage{algorithm}
\usepackage{algpseudocode}

\usepackage{pgfplots}
\pgfplotsset{compat=1.17}
\usepgfplotslibrary{fillbetween}

\definecolor{darkblue}{rgb}{0, 0, 0.5}
\hypersetup{colorlinks=true, citecolor=darkblue, linkcolor=darkblue, urlcolor=darkblue}

\definecolor{cGRPO}{HTML}{8A8F98}
\definecolor{cGSPO}{HTML}{3B4252}
\definecolor{cRRR}{HTML}{2F9E6E}
\definecolor{cTIS}{HTML}{E08A2E}
\definecolor{cTISR}{HTML}{18A0A8}
\definecolor{cDPPO}{HTML}{8B5CF6}
\definecolor{cSATR}{HTML}{1C5CAB}
\definecolor{cSATRR}{HTML}{C0392B}
\definecolor{cSienna}{HTML}{A0522D}
\definecolor{cAccent}{HTML}{2B6CB0}
\definecolor{cCollapse}{HTML}{8B0000}

\theoremstyle{plain}
\newtheorem{theorem}{Theorem}[section]
\newtheorem{lemma}[theorem]{Lemma}

\theoremstyle{definition}
\newtheorem{proposition}[theorem]{Proposition}

\newtcolorbox{takeaway}[1][]{colback=metabg!6,colframe=metabg!70,arc=2.5pt,
  boxrule=0.6pt,left=2mm,right=2mm,top=1.2mm,bottom=1.2mm,#1}

\newcommand{\elow}{\varepsilon_{\mathrm{low}}}
\newcommand{\ehigh}{\varepsilon_{\mathrm{high}}}
\newcommand{\dtv}{D_{\mathrm{TV}}}
\newcommand{\dpi}{\bar d_\pi}
\newcommand{\clip}{\operatorname{clip}}
\newcommand{\sg}{\operatorname{sg}}
\newcommand{\pos}[1]{[#1]_{+}}
\newcommand{\J}{\mathcal{J}}
\newcommand{\I}{\mathcal{I}}

\makeatletter
\DeclareRobustCommand\onedot{\futurelet\@let@token\@onedot}
\def\@onedot{\ifx\@let@token.\else.\null\fi\xspace}

\makeatother

\usepackage{amsmath,amsfonts,bm}

\def\eqref#1{equation~\ref{#1}}

\def\1{\bm{1}}

\DeclareMathAlphabet{\mathsfit}{\encodingdefault}{\sfdefault}{m}{sl}
\SetMathAlphabet{\mathsfit}{bold}{\encodingdefault}{\sfdefault}{bx}{n}

\newcommand{\E}{\mathbb{E}}

\usepackage{amsmath,amsfonts,bm}

\def\eqref#1{equation~\ref{#1}}

\def\1{\bm{1}}

\DeclareMathAlphabet{\mathsfit}{\encodingdefault}{\sfdefault}{m}{sl}
\SetMathAlphabet{\mathsfit}{bold}{\encodingdefault}{\sfdefault}{bx}{n}

\usepackage{multirow}
\usepackage{diagbox}
\usepackage{makecell}
\usepackage{tabularx}
\usepackage{graphicx}

\usepackage{array}
\usepackage{rotating}

\definecolor{aliceblue}{rgb}{0.94, 0.97, 1.0}
\definecolor{citecolor}{HTML}{0071BC}
\definecolor{linkcolor}{HTML}{ED1C24}
\definecolor{darkgreen}{HTML}{539165}

\makeatletter
\newcommand{\thickhline}{%
 \noalign {\ifnum 0=`}\fi \hrule height 1pt
 \futurelet \reserved@a \@xhline
}
\makeatother

\renewcommand{\paragraph}[1]{\vspace{1.25mm}\noindent\textbf{#1}}
\title{\centering Stale but Stable: Staleness-Adaptive Trust Regions for Stabilizing Asynchronous Reinforcement Learning}

\author[1,2,*]{Junyao Yang}
\author[1,\dag]{Yucheng Shi}
\author[1,3]{Zongxia Li}
\author[1,4]{Zhongzhi Li}
\author[1,5]{Ruhan Wang}
\author[6]{Xiangxin Zhou}
\author[1]{Kishan Panaganti}
\author[1]{Haitao Mi}
\author[1]{Leowei Liang}

\affiliation[1]{Tencent Hy LLM Frontier}
\affiliation[2]{National University of Singapore}
\affiliation[3]{University of Maryland, College Park}
\affiliation[4]{University of Georgia}
\affiliation[5]{Indiana University}
\affiliation[6]{Tencent Hy}

\contribution[\dag]{Project Lead}
\contribution[*]{Corresponding Author}

\email{junyaoyang@u.nus.edu, yuchengshi@global.tencent.com}
\headercontent{{\small Project Page: \href{https://jyyang26.github.io/stable_async_analysis}{https://jyyang26.github.io/stable\_async\_analysis}}}

\abstract{
Asynchronous reinforcement learning improves throughput by decoupling rollout generation from optimization, but the resulting staleness is an inevitable byproduct, compounded jointly by policy lag, engine delays, and mixture-of-experts routing. From a trust-region perspective, this mismatch is critical: in the finite-horizon improvement bound, training-inference divergence governs the approximation error, whereas PPO clipping only gates sampled outward updates and therefore acts as a sampled surrogate rather than a full-policy constraint. As a result, the high-staleness update can remain weakly controlled in exactly the asynchronous regime where stale rollouts matter most.

We introduce the \textbf{S}taleness-\textbf{A}daptive \textbf{T}rust Region (\textbf{\textsc{SAT}}), which uses the detached sampled log-ratio as a practical staleness proxy, identifies the high-mismatch tail within each batch through \textbf{Staleness-based kernel function scaling}, and contracts only the sign-selected endpoint of the nominal PPO interval using \textbf{Effective contraction factors}. 
This design preserves the baseline behavior on ordinary tokens, while making the update more conservative exactly on newly intercepted outward bands. 
We evaluate \textsc{SAT} in a fully decoupled asynchronous reinforcement learning setup built on Qwen3-30B-A3B-Base, leveraging SGLang as the inference engine and Megatron as the training pipeline. 
In this setting, \textsc{SAT}-GSPO w/ R3 attains the best observed AIME24 avg@8, reaching $35.83$ at lag $1$ and $34.79$ at lag $8$, while \textsc{SAT}-GSPO reaches $34.17$ at lag $1$. 
More broadly, the results indicate that aligning the clip interval with observed staleness heterogeneity is an effective way to stabilize the reported asynchronous regime.
}

\date{\today}

\begin{document}
\thispagestyle{firstheader}
\maketitle

\section{Introduction}
\label{sec:intro}
 
Asynchronous reinforcement learning serves as a core technique for training large foundation models, accelerating training speeds and maximizing GPU utilization. However, this asynchronous paradigm simultaneously suffers from training-inference mismatch and high staleness, both of which can significantly degrade the resulting model performance.
Shown in Figure ~\ref{fig:pipeline}, for the batch-wise asynchronous pipeline in which rollout and optimization run on decoupled resources and weights are broadcast every several trainer steps~\citep{zhu2025slime,zheng2023sglang,shoeybi2019megatron}. 
We call this interval step \emph{lag}, and the starting observation is that lag is only a coarse systems label. 
What the loss actually sees is the staleness between the behavior policy that generated token and the target policy that optimizes it, summarized by the log importance ratio between target and behavior policy, and in modern deployments this mismatch mixes policy-update lag with implementation effects from engine kernels, low-precision numerics, and mixture-of-experts routing. Its distribution is heterogeneous and long-tailed even within a single rollout window. 
We call this failure mode \textbf{Staleness-associated Asynchronous RL instability}: fully illustrated in Section~\ref{sec:bound}, as the staleness grows, with training-inference divergence inflates and fixed-clip asynchronous updates become increasingly brittle, ultimately threatening late-stage training stability.

In the finite-horizon policy-improvement bound, the approximation penalty is a cumulative total-variation divergence, and state-wise target-behavior policy divergence is exactly a behavior expectation of the absolute target-behavior policy ratio deviation; a hard ratio bound on \emph{all} actions would therefore control the penalty, with detailed derivation shown in Section~\ref{sec:ppo-soft}. PPO clipping imposes no such bound: it is a sampled-objective device that suppresses outward gradients beyond a nominal boundary, preserves pull-back updates, and inspects nothing beyond the one sampled action. Within that surrogate, a fixed radius treats every token as equally reliable with exactly the wrong prior under heterogeneous asynchronous mismatch, where a uniformly small radius starves learning on the bulk of tokens while a uniformly large one grants the high-staleness tail the same outward allowance as reliable data.  Yet the clip radius is precisely the handle that decides where outward gradients stop, motivating the central question of the paper: \textbf{Can the clip radius itself adapt to the staleness so that asynchronous RL becomes more stable exactly on the high-staleness scenerio?}
 
\textbf{From the perspective of the finite-horizon LLM policy-improvement bound, we arrive at a simple design principle:} \textbf{{to suppress future training--inference divergence growth, the clip radius should contract precisely on the observed high-staleness tail rather than remain fixed for every token.}}
We therefore propose \satrfull{} (\textbf{\satr}), which adapt the clip contraction to the observed staleness. 
\satr{} first treats the detached sampled log-ratio as a practical per-token staleness proxy, identifies the high-staleness tail of each batch through a self-calibrating quantile, and maps unusually large observed staleness to a tighter outward allowance through \textbf{Staleness-based kernel function scaling} while leaving ordinary tokens and pull-back updates at the baseline behavior. 
\satr{} then contracts only the sign-selected endpoint of PPO's nominal clip interval on those tokens through \textbf{Effective contraction factors}, thereby cutting off outward high-staleness updates earlier on the side that would locally enlarge the divergence penalty.

 
Across extensive asynchronous experiments under identical setting, we find that \satr{} achieves better performance in high-staleness scenarios, and that the top-performing \satr{}-GSPO w/ R3 variant also maintains comparatively low observed staleness. The two mechanisms are complementary in the reported runs: routing replay governs the level of the logged mismatch, while \satr{} changes where divergence gradients stop.
This paper makes three contributions:
\begin{itemize}
\item \textbf{Problem formalization.} We formalize batch-wise asynchronous RL, separate the configured lag from the realized sampled mismatch $d_{b,t}$ and its policy-update and implementation components, document the mismatch-inflation and evaluation-collapse signatures that emerge as lag grows, and make precise, via the finite-horizon policy-improvement bound, what a hard all-action ratio envelope would control and what the sampled clipped objective actually optimizes.
\item \textbf{Method and analysis.} We introduce \satr{}, a direction-aware, staleness-adaptive contraction of PPO's nominal clip interval, and characterize its update geometry exactly: the adaptive interval is contained in PPO's, the surrogate is pointwise pessimistic, and the two objectives differ only on one sign-selected outward band per gated token.
\item \textbf{Asynchronous evaluation.} We conduct controlled experiments at configured lags $1$ and $8$, showing that \satr{} achieves better performance in high-staleness scenarios and that the top-performing \satr{} variant also maintains comparatively low observed staleness. We position \satr{} within the asymmetric-clipping family, where it differs from PPO's fixed radius and DPPO's behavior-probability-weighted boundary by setting the active boundary from the observed staleness of the current batch.
\end{itemize}
\section{Async RL Preliminaries and Staleness Symptoms}
\label{sec:tv}

\subsection{Asynchronous RL and realized mismatch}
\label{sec:setup}

A prompt $x$ is answered by a response $y_b=(y_{b,1},\dots,y_{b,T_b})$ sampled autoregressively; the state at position $t$ is $s_{b,t}=(x,y_{b,1},\dots,y_{b,t-1})$. Rewards are sparse and sequence-level, $R(y)\in[-\xi,\xi]$, and $\hat A_{b,t}$ denotes the advantage estimate, taking the GRPO group as an example. The \emph{behavior policy} $\mu_{b,t}$ is the distribution the rollout engine actually sampled token $(b,t)$ from; the \emph{train policy} $\pi^{(j)}$ is the policy at the current trainer update $j$. The token-level importance ratio and per-token log-ratio are:
\begin{equation}
r_{b,t}\;=\;\frac{\pi^{(j)}(y_{b,t}\mid s_{b,t})}{\mu_{b,t}(y_{b,t}\mid s_{b,t})},
\qquad
d_{b,t}\;\triangleq\;\log r_{b,t},
\label{eq:ratio-def}
\end{equation}
and we write $\mathcal T_{\mathrm{act}}$ for the active (non-padding) token positions of one batch. The per-state total variation is:
\[
\dtv(\mu\,\|\,\pi)[s]=\tfrac12\sum_a|\mu(a\mid s)-\pi(a\mid s)|.
\]

\begin{figure}[htbp]
\centering
\includegraphics[width=0.85\linewidth]{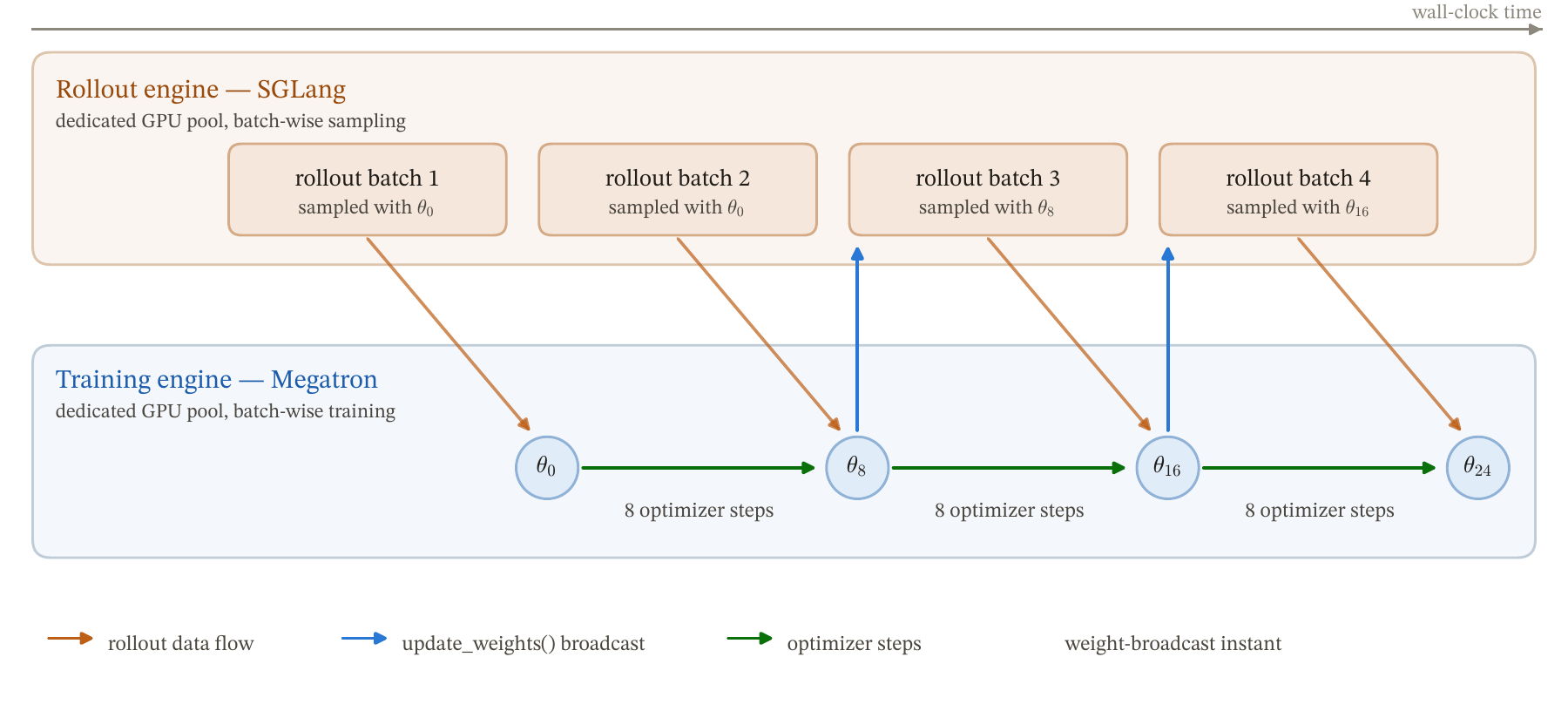}
\caption{\textbf{Batch-wise asynchronous RL.} Rollout and optimization run on decoupled resources with periodic weight broadcast, so one in-flight batch is typically trained at a configured lag of $n$ policy versions. This is the setting behind $\mu_{b,t}$, $\pi^{(j)}$, and $d_{b,t}$. 
With full notation table shown in Appendix~\ref{app:notation},
synchronous RL and fully Asynchronous RL workflows can be found in Appendix~\ref{app:sync-rl} and Appendix~\ref{app:fully-async}, respectively.}
\label{fig:pipeline}
\end{figure}

\FloatBarrier

\paragraph{Configured lag versus realized mismatch in batch-wise asynchronous RL.}
Our experiments use the batch-wise asynchronous regime of Figure~\ref{fig:pipeline}: weights are broadcast every $n$ trainer steps, and we call $n$ the \emph{configured lag}. Write $d_{b,t}=\log r_{b,t}$ for the observed sampled mismatch, let $\pi^{(\ell)}$ be the trainer-side reference log-probability at the rollout checkpoint, and define $\Delta_{b,t}^{\pi}=\log\pi^{(j)}-\log\pi^{(\ell)}$ and $\Delta_{b,t}^{\mathrm{impl}}=\log\pi^{(\ell)}-\log\mu_{b,t}$ at $(y_{b,t},s_{b,t})$. Then:
\begin{equation}
d_{b,t}
\;=\;\log r_{b,t}
\;=\;\Delta_{b,t}^{\pi}
\;+
\;\Delta_{b,t}^{\mathrm{impl}},
\label{eq:split}
\end{equation}
where $\Delta_{b,t}^{\pi}$ is the policy-update mismatch accumulated over the $n$ optimizer steps of the window and $\Delta_{b,t}^{\mathrm{impl}}$ is the implementation mismatch between the trainer-side reference and the actual rollout engine. The first term $\Delta_{b,t}^{\pi}$ is the object PPO's ratio machinery is designed to correct. 
$\Delta_{b,t}^{\mathrm{impl}}$ can remain nonzero even at identical weights because the two engines differ in kernels, numerics, and MoE routing. Engine, router, and hardware effects are useful \emph{diagnostic hypotheses} inside $\Delta_{b,t}^{\mathrm{impl}}$, but without additional intermediate distributions they do not form a unique additive decomposition; Appendix~\ref{app:sources} discusses them in that spirit. The operational consequence is the one stated in the introduction: the configured lag does not determine the realized distribution of $d_{b,t}$, so any mechanism keyed directly to $n$ addresses only part of the problem.

\subsection{Staleness-associated training instability}
\label{sec:bound}

Our experiments point to a crucial background fact: once the batch-wise async lag becomes large enough, the main issue is no longer only a looser lower bound but an observable loss of training stability. In the controlled comparisons below, all runs share optimizer, data, reward, and model; only the configured lag differs. We therefore read these figures as background evidence for what async RL must explain, not as a theorem about a unique causal threshold.

\begin{figure}[htbp]
\centering
\begin{subfigure}[t]{0.495\linewidth}
\centering
\includegraphics[width=\linewidth]{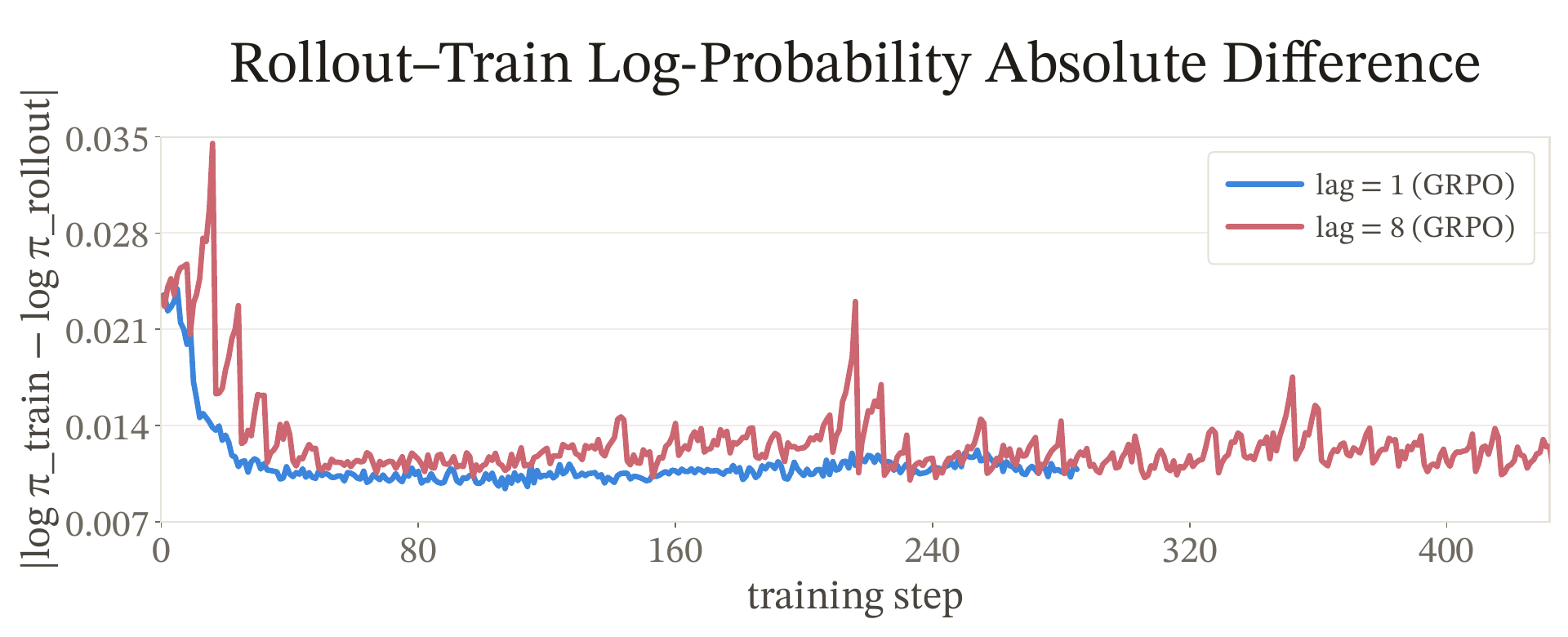}
\caption{Training--inference mismatch $\dpi$.}
\label{fig:collapse-dpi}
\end{subfigure}\hfill
\begin{subfigure}[t]{0.495\linewidth}
\centering
\includegraphics[width=\linewidth]{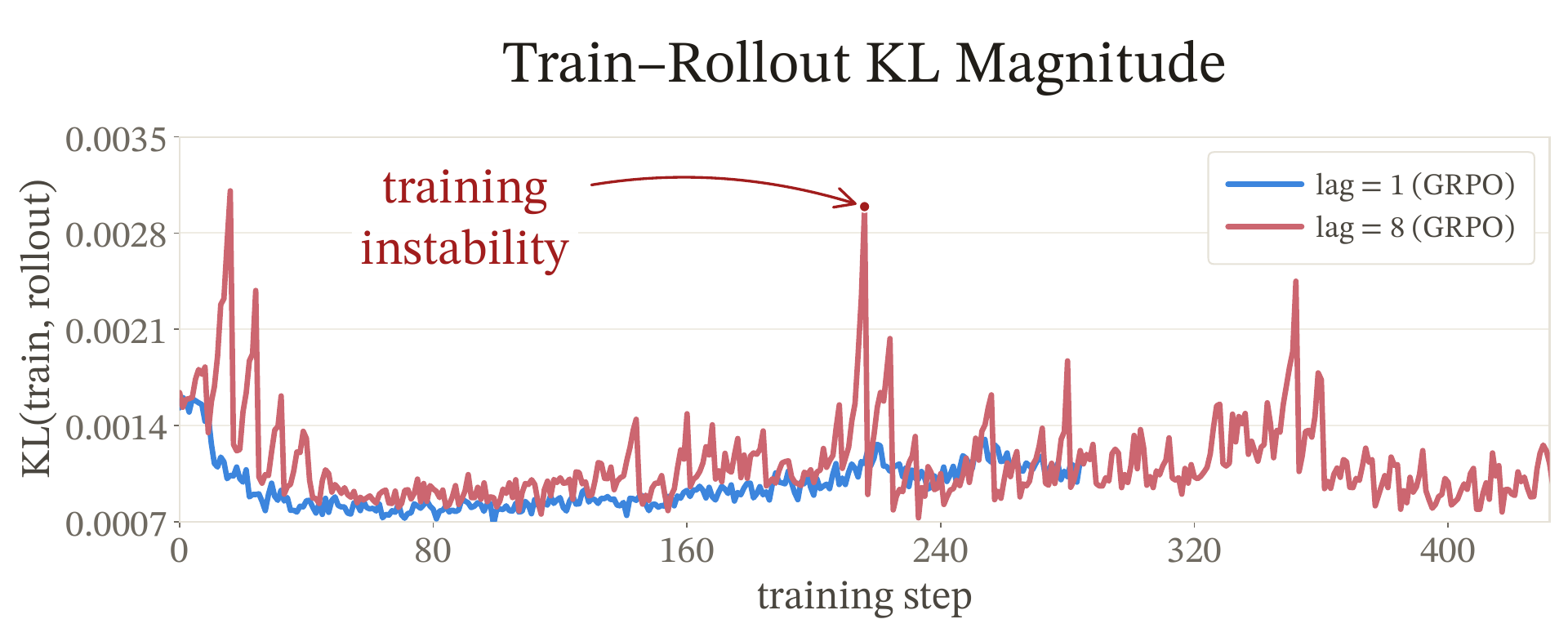}
\caption{Train--rollout $\mathrm{KL}$.}
\label{fig:collapse-kl}
\end{subfigure}
\caption{\textbf{Background instability diagnostics under a variance method at configured lag $1$ and $8$.} This figure tracks two batch-averaged mismatch diagnostics of a variance method across training under two configured lag settings, with the two curves in each panel distinguishing the mild-lag and high-lag runs under an otherwise identical setup. The panels also mark where the high-lag run first leaves its low-mismatch baseline band, together with the numerical range covered by each axis.}
\label{fig:collapse-diagnostics}
\end{figure}

\begin{figure}[htbp]
\centering
\includegraphics[width=0.85\linewidth]{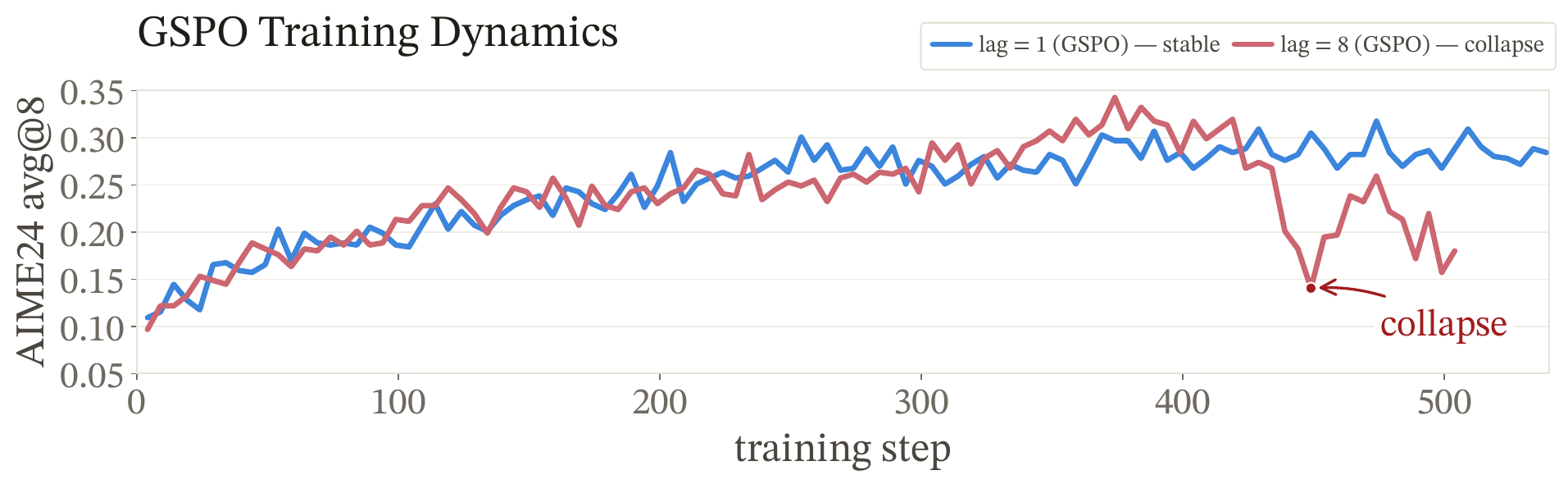}
\caption{\textbf{Background evaluation collapse under a variance method at configured lag $1$ and $8$.} This figure plots the evaluation benchmark score of a variance method across training under two configured lag settings, where the collapse step of the high-lag run is explicitly annotated with a point and label. The panel is used to align the timing of evaluation collapse with the mismatch signals of Figure~\ref{fig:collapse-diagnostics}.}
\label{fig:collapse-aime}
\end{figure}


\textbf{Two instability symptoms recur under high-lag asynchronous training.} Figure~\ref{fig:collapse-diagnostics} and Figure~\ref{fig:collapse-aime} shows how batch-wise asynchronous training deteriorate as the lag grows, from which three observations stand out.
\begin{itemize}
    \item \textbf{Mismatch inflation under increased staleness.} Mismatch $\dpi$ and its KL counterpart stay in a narrow band at lag $1$ but develop heavy-tailed spikes at lag $8$, where $\dpi$ exceeds $0.03$ and the KL diagnostic exceeds $2.8\times 10^{-3}$, indicating a widening of the ratio distribution rather than a shift of its center.
    \item \textbf{Training collapse due to increase of staleness.} The lag-$1$ run tracks the evaluation-benchmark band near $0.30$, whereas the lag-$8$ run peaks near $0.34$ before falling to about $0.14$ by step $449$. 
\end{itemize}

\section{LLM Policy Improvement and Ratio Control in Async RL}
\label{sec:ppo-soft}

\subsection{Finite-horizon LLM policy-improvement surrogate}
\label{sec:llm-pi}

Consider an undiscounted, finite-horizon generation problem with $|R(y)|\le\xi$, behavior policy $\mu$, and target policy $\pi$, and define $\J(\pi)=\E_{y\sim\pi}[R(y)]$. We assume trajectory support compatibility: at every $\mu$-reachable state used below, $\pi(\cdot\mid s)\ll\mu(\cdot\mid s)$. When necessary, an end-of-sequence action and padding make the horizon fixed. Following the Kakade--Langford line of analysis~\citep{kakade2002approximately,schulman2015trpo}, define the linear surrogate:
\begin{equation}
L'_{\mu}(\pi)\;=\;\E_{y\sim\mu}\!\left[R(y)\sum_{t=1}^{T}\Big(\frac{\pi(y_t\mid s_t)}{\mu(y_t\mid s_t)}-1\Big)\right].
\label{eq:surrogate}
\end{equation}
A representative finite-horizon lower bound~\citep{qi2026dppo} is:
\begin{equation}
\J(\pi)-\J(\mu)\;\geq\;L'_{\mu}(\pi)\;-
4\xi\,\E_{y\sim\mu}\!\left[\sum_{t=1}^{T}\dtv\big(\mu(\cdot\mid s_t)\,\|\,\pi(\cdot\mid s_t)\big)\right].
\label{eq:tvbound}
\end{equation}
The penalty is an expected \emph{cumulative} token-level divergence, not a length-normalized average. Equation~\eqref{eq:tvbound} is the policy-improvement statement relevant for the rest of the paper: smaller statewise divergence makes the approximation term less pessimistic, but it does not by itself imply that an arbitrary smaller update is automatically better. Non-decrease still depends on the surrogate gain dominating the divergence penalty.


\subsection{Ratio envelopes and the sampled clip in async RL}
\label{sec:ratio-envelope}

The relation between ratio control and Section~\ref{sec:llm-pi}'s policy-improvement surrogate is exact but conditional on support compatibility.

\begin{lemma}[\textbf{Ratio deviation as a statewise divergence identity}]
\label{lem:tv-ratio}
Fix a state $s$ and suppose $\pi(\cdot\mid s)\ll\mu(\cdot\mid s)$. Then, with $r_s(a)=\pi(a\mid s)/\mu(a\mid s)$,
\begin{equation}
\dtv\big(\mu(\cdot\mid s)\,\|\,\pi(\cdot\mid s)\big)
\;=\;\frac12\,\E_{a\sim\mu(\cdot\mid s)}\big|r_s(a)-1\big|.
\label{eq:tvid}
\end{equation}
\end{lemma}
\textbf{Proof.}
$\frac12\sum_{a}|\pi(a\mid s)-\mu(a\mid s)|=\frac12\sum_{a}\mu(a\mid s)\,|r_s(a)-1|$; absolute continuity rules out unaccounted target mass on actions with $\mu(a\mid s)=0$.

Full-vocabulary softmax sampling satisfies the support condition in exact arithmetic, which is why our experiments in Section~\ref{sec:exp-setup} disable nucleus truncation. However, top-$k$ or top-$p$ truncation can violate this condition and requires the separate treatment discussed in Appendix~\ref{app:stabilizers}.

Equation~\eqref{eq:tvid} is the step that connects a hard ratio envelope to the policy-improvement bound of Equation~\eqref{eq:tvbound}. If, at a fixed state, every action satisfies a hard ratio bound, then the divergence term in Equation~\eqref{eq:tvbound} is correspondingly controlled:
\begin{proposition}[\textbf{Ideal all-action ratio region}]
\label{prop:hard-ratio}
Under the conditions of Lemma~\ref{lem:tv-ratio}, if:
\begin{equation}
|r_s(a)-1|\;\le\;\varepsilon\qquad\text{for $\mu$-almost every action $a$,}
\label{eq:all-action}
\end{equation}
then $\dtv(\mu\,\|\,\pi)[s]\le\varepsilon/2$. More generally, if $1-\varepsilon_{\mathrm{low}}(a)\le r_s(a)\le 1+\varepsilon_{\mathrm{high}}(a)$ for every action, then:
\begin{equation}
\dtv(\mu\,\|\,\pi)[s]\;\le\;\frac12\,\E_{a\sim\mu}\max\big\{\varepsilon_{\mathrm{low}}(a),\,\varepsilon_{\mathrm{high}}(a)\big\},
\label{eq:asym-tv}
\end{equation}
which for constant asymmetric radii is at most $\tfrac12\max\{\elow,\ehigh\}$.
\end{proposition}
\textbf{Proof.}
Each assumption gives the corresponding pointwise bound on $|r_s(a)-1|$, substitution into Equation~\eqref{eq:tvid} proves each claim.

This proposition describes an \emph{ideal all-action feasible set}. The quantifier is essential: a bound on the one sampled action says nothing about the rest of the vocabulary distribution. That gap is exactly why sampled clipping in async RL should be read as a surrogate motivated by the hard-envelope argument above, not as the hard envelope itself.

\section{\satr: From Heterogeneous Mismatch to an Adaptive Clip}
\label{sec:satr}



\subsection{From Policy Improvement to a Staleness-Adaptive Radius}
\label{sec:satr-motivation}
Section~\ref{sec:ppo-soft} ends one step short of a method: shown in Lemma~\ref{lem:tv-ratio}, statewise divergence controls the finite-horizon approximation term, that divergence is exactly a behavior expectation of absolute ratio deviation, and a hard all-action ratio envelope would therefore control the penalty, as discussed in Proposition~\ref{prop:hard-ratio}, but nothing yet says what the \emph{sampled} objective should do when the envelope is unavailable. To make the target of that chain explicit, we name once the two LLM-regime divergence quantities:
\begin{equation}
\begin{gathered}
D_{\mathrm{TV}}^{\max}(\mu\|\pi)\triangleq \max_{s_t} D_{\mathrm{TV}}\bigl(\mu(\cdot\mid s_t)\|\pi(\cdot\mid s_t)\bigr),\\
D_{\mathrm{TV}}(\mu,\pi)\triangleq \mathbb{E}_{y\sim\mu}\!\left[\sum_{t=1}^{|y|} D_{\mathrm{TV}}\bigl(\mu(\cdot\mid s_t)\|\pi(\cdot\mid s_t)\bigr)\right].
\end{gathered}
\label{eq:satr-divergences}
\end{equation}
In finite-horizon LLM setting with horizon $T$, $\gamma=1$, and $\xi=\max_y |R(y)|$, the improvement gap admits both:
\begin{align}
\J(\pi)-\J(\mu)
&\ge L'_{\mu}(\pi)-2\xi\,T(T-1)\,\bigl(D_{\mathrm{TV}}^{\max}(\mu\|\pi)\bigr)^{2},\label{eq:satr-maxbound}\\
\J(\pi)-\J(\mu)
&\ge L'_{\mu}(\pi)-4\xi\,D_{\mathrm{TV}}(\mu,\pi).
\label{eq:satr-avgbound}
\end{align}
The first is the finite-horizon analogue of a max-divergence trust-region bound; the second restates Equation~\eqref{eq:tvbound} in the named notation and is the cumulative token-level form that a token-wise sampled objective can actually influence on long responses. What PPO optimizes, however, is neither bound but the sampled clipped surrogate:
\begin{equation}
\mathcal S^{\mathrm{PPO}}_{b,t}
=\min\!\Big(r_{b,t}\hat A_{b,t},\;\bar r^{\mathrm{PPO}}_{b,t}\hat A_{b,t}\Big),
\qquad
\bar r^{\mathrm{PPO}}_{b,t}=\clip\!\big(r_{b,t},\,1-\elow,\,1+\ehigh\big),
\label{eq:ppo-objective}
\end{equation}
whose only guardrail is the fixed pair $(\elow,\ehigh)$: the clipped branch suppresses outward gradients beyond the nominal boundary, preserves pull-back updates, and inspects nothing beyond the one sampled action.
 
\paragraph{Why a fixed radius is an imperfect default.}
A fixed $\varepsilon$ treats every sampled token as equally reliable. In the batch-wise asynchronous pipeline of Section~\ref{sec:tv}, however, the realized mismatch $d_{b,t}$ mixes policy-update and implementation terms in Equation~\eqref{eq:split} and is heterogeneous within a single rollout window: most tokens remain close to the behavior policy, while a small tail is affected by policy lag, training--inference engine differences, router choices, or numerical effects. A uniformly tiny radius therefore suppresses useful learning on the bulk of the data, while a uniformly large radius grants the high-staleness tail the same outward allowance as reliable tokens. Read against Equations~\eqref{eq:satr-maxbound}--\eqref{eq:satr-avgbound}, both defaults miss the point: the divergence penalty is driven by where the mismatch is large, and no single global radius can be simultaneously permissive on the bulk and conservative on the tail.

\paragraph{A sampled proxy, not a divergence estimate.}
The observable that exposes this tail is the sampled log-ratio $d_{b,t}=\log r_{b,t}$, which is already computed in off-policy training. It must not, however, be interpreted as $D_{\mathrm{TV}}$: the sampled action's contribution to the $D_{\mathrm{TV}}$ sum carries a behavior-probability weight,
\begin{equation}
\tfrac12\,\mu(a\mid s)\,\big|r(a\mid s)-1\big|\;=\;\tfrac12\,\big|\pi(a\mid s)-\mu(a\mid s)\big|,
\label{eq:tv-mass}
\end{equation}
so a rare action can have an enormous ratio while moving little probability mass. As detailed in Section~\ref{sec:limitations}, its use as a per-token risk score is an empirical design choice.

\paragraph{Design of SAT.}
These observations govern how an adaptive clipping rule should operate. Because training only observes sampled actions rather than the full distribution, the rule cannot enforce a strict, all-action trust region; instead, it must function at the same surrogate objective level as Equation~\eqref{eq:ppo-objective}.
Specifically, SAT contracts PPO's nominal interval only on tokens whose observed mismatch is unusually large {for the current batch}, and only on the side that moves the sampled ratio farther from one, with the direction that locally enlarges the divergence penalty. Consequently, it leaves ordinary tokens and pull-back updates untouched, reducing exactly to the baseline objective when disabled.

\subsection{Staleness-Adaptive Trust Region}
\label{sec:satr-method}

\textbf{{SAT targets a simple but useful goal, which reshapes the sampled PPO surrogate so that tokens with unusually large observed mismatch receive a tighter outward allowance, while ordinary tokens and pull-back updates retain the baseline behavior}}. The whole method is defined by the sequence of formulas below.

\textbf{Detached sampled token-level staleness.} \satr{} begins from the sampled mismatch signal already available in asynchronous training and treats the detached token-level log-ratio as a practical staleness proxy for the current sample:
\begin{equation}
d_{b,t}=\log r_{b,t},\qquad d^{\mathrm{sg}}_{b,t}\triangleq\sg[d_{b,t}],
\label{eq:satr-d}
\end{equation}
where $\sg$ denotes stop-gradient. This detached score is not itself optimized; it only decides how conservative the clip interval should be on the current sample.

\textbf{Staleness-based kernel function scaling.} Rather than compare $|d_{b,t}|$ with a fixed global threshold, \satr{} calibrates the high-staleness tail relative to the current batch through the detached empirical quantile:
\begin{equation}
\widehat F_{|d|}(x)=\frac{1}{|\mathcal T_{\mathrm{act}}|}\sum_{(b,t)\in\mathcal T_{\mathrm{act}}}\mathbf 1\{|d_{b,t}|\le x\},
\qquad
q=\sg\Big[\inf\{x\ge 0:\widehat F_{|d|}(x)\ge \alpha\}\Big].
\label{eq:quantile}
\end{equation}
The role of $q$ is purely referential: it marks where the current batch begins to enter its mismatch tail, where $\alpha$ denotes the target quantile threshold, which set to 0.90 in our implementation. Because $q$ is recomputed from the same active tokens being optimized, the method adapts to rescaling of the mismatch profile without introducing another hand-tuned absolute cutoff.
\satr{} then maps the sign-separated mismatch magnitudes through a monotone kernel:
\begin{equation}
\psi(u;q)=\frac{1}{1+(u/q)^{2}},
\qquad
u_{+,b,t}=\pos{d^{\mathrm{sg}}_{b,t}},\quad u_{-,b,t}=\pos{-d^{\mathrm{sg}}_{b,t}},
\label{eq:kernel}
\end{equation}
where $\psi(0;q)=1$, $\psi(q;q)=\tfrac12$, and $\partial\psi/\partial u\le 0$. The sign split is the key structural choice: when the sampled ratio is already above $1$, the adaptive rule should only tighten the upper side of the interval; when it is below $1$, it should only tighten the lower side. Otherwise the method would also suppress pull-back moves that already head back toward the behavior policy.

\textbf{Effective contraction factors.} The contraction is applied only on the high-staleness tail:
\begin{equation}
c_{\pm,b,t}=
\begin{cases}
\psi(u_{\pm,b,t};q), & q>0\ \text{and}\ |d^{\mathrm{sg}}_{b,t}|>q,\\[2pt]
1, & \text{otherwise},
\end{cases}
\label{eq:gate}
\end{equation}
which yields the adaptive radii:
\begin{equation}
\tilde\varepsilon_{\mathrm{low},b,t}=\elow\,c_{-,b,t},
\qquad
\tilde\varepsilon_{\mathrm{high},b,t}=\ehigh\,c_{+,b,t}.
\label{eq:adaptive-eps}
\end{equation}
Since $0<c_{\pm,b,t}\le 1$, these radii never expand the original PPO interval. Ungated tokens recover the baseline radii exactly, while gated positive-mismatch tokens only shrink the upper side and gated negative-mismatch tokens only shrink the lower side.
\begin{itemize}
\item \textbf{SAT core mechanism.} \satr{} is conservative only where the sampled mismatch is unusually large and only on the side that would move the sampled ratio farther away from one. The explicit $q>0$ branch avoids division by zero when most mismatch values vanish, and under the inverse-cumulative distribution function convention the strict gate $|d|>q$ selects at most the top decile of active positions.
\end{itemize}

\textbf{SAT sampled objective.} \satr{} then obtains the final sampled surrogate by dropping the adaptive radii into the original clipped objective:
\begin{align}
\bar r^{\satr}_{b,t}&=\clip\!\big(r_{b,t},\,1-\tilde\varepsilon_{\mathrm{low},b,t},\,1+\tilde\varepsilon_{\mathrm{high},b,t}\big),
\label{eq:satr-clip}\\
\mathcal S^{\satr}_{b,t}&=\min\!\big(r_{b,t}\hat A_{b,t},\;\bar r^{\satr}_{b,t}\hat A_{b,t}\big).
\label{eq:satrloss}
\end{align}
This substitution shows that \satr{} is not a new objective family but a drop-in replacement for PPO clipping: it replaces $\bar r^{\mathrm{PPO}}_{b,t}$ by $\bar r^{\satr}_{b,t}$ inside the same sampled surrogate. When $c_{\pm,b,t}=1$, the underlying base objective is recovered exactly. Because the sign-selected endpoint moves inward, the adaptive interval remains contained in PPO's nominal interval, while the untouched side matches PPO. Pull-back branches therefore stay unchanged, and only the outward band between the adaptive and nominal boundaries is newly clipped.

\paragraph{Sequence-level objectives.}
Equations~\eqref{eq:satr-d}--\eqref{eq:satrloss} describe the token-ratio implementation. For the reported GSPO extension, the scalar in the clipped objective and in the \satr{} score is instead the length-normalized sequence ratio:
\begin{equation}
\rho^{\mathrm{seq}}_{b}=\exp\!\Big(\frac{1}{T_b}\sum_{t=1}^{T_b}\log r^{\mathrm{tok}}_{b,t}\Big),
\qquad
r^{\mathrm{tok}}_{b,t}=\frac{\pi^{(j)}(y_{b,t}\mid s_{b,t})}{\mu_{b,t}(y_{b,t}\mid s_{b,t})},
\label{eq:gspo-ratio}
\end{equation}
broadcast to the response's active token positions before the same active-token quantile and objective aggregation are applied. The empirical quantile therefore length-weights sequences, and every token of one response receives the same sequence-ratio score. 

\subsection{Update Geometry of \satr}
\label{sec:satr-mechanism}
 
The construction of Section~\ref{sec:satr-method} admits an exact local characterization, at the same surrogate level as Section~\ref{sec:ratio-envelope}: \satr{} does not impose a hard all-action trust region; its entire effect on the sampled objective is confined to one sign-selected outward band per gated token. The following proposition collects the interval geometry, the pointwise ordering, and the exact locus of change relative to Equation~\eqref{eq:ppo-objective}.
 
\begin{proposition}[\textbf{Update geometry of \satr}]
\label{prop:geometry}
\label{prop:containment}
\label{prop:pessimism}
Fix a sampled token with ratio $r$, advantage $\hat A$, and contraction factors $c_{\pm,b,t}$ from Equation~\eqref{eq:gate}, and define
$\I^{\mathrm{PPO}}=[1-\elow,\,1+\ehigh]$ and
$\I^{\satr}_{b,t}=[1-\elow c_{-,b,t},\,1+\ehigh c_{+,b,t}]$
for positive base radii. Then:
(i)~\emph{Containment:} $\I^{\satr}_{b,t}\subseteq\I^{\mathrm{PPO}}$; the intervals coincide when the gate is off, and when it fires, the endpoint selected by $\operatorname{sign}(d_{b,t})$ moves strictly inward while the opposite endpoint is unchanged.
(ii)~\emph{Pointwise pessimism:} $\mathcal S^{\satr}(r,\hat A)\le\mathcal S^{\mathrm{PPO}}(r,\hat A)$.
(iii)~\emph{Exact locus of change:} holding the stopped clip limits fixed and away from the clipping kinks, the derivatives with respect to $r$ differ exactly on the newly clipped outward bands
\begin{equation}
\{\hat A>0,\ 1+\ehigh c_{+,b,t}<r<1+\ehigh\}
\;\vee\;
\{\hat A<0,\ 1-\elow<r<1-\elow c_{-,b,t}\},
\label{eq:bands}
\end{equation}
where \satr{} has zero derivative and PPO retains its outward derivative $\hat A$. Pull-back branches are unchanged, and ratios already outward-clipped by PPO have zero derivative under both objectives.
\end{proposition}
 
\textbf{Proof.}
(i) Equation~\eqref{eq:gate} gives $0<c_{\pm,b,t}\le 1$, so multiplying a positive base radius by $c_{\pm,b,t}$ can only move the corresponding boundary toward $1$ or leave it fixed; strictness on the gated side follows from $\psi(u;q)<1$ at $u>q>0$.
(ii)--(iii) For $\hat A>0$, compare $\mathcal S^{\mathrm{PPO}}=\hat A\min\{r,1+\ehigh\}$ with $\mathcal S^{\satr}=\hat A\min\{r,1+\ehigh c_{+,b,t}\}$: since $c_{+,b,t}\le 1$, the adaptive boundary is no larger, which gives the ordering, and the two functions and their derivatives differ exactly on $1+\ehigh c_{+,b,t}<r<1+\ehigh$. The case $\hat A<0$ is symmetric with $\max$ and the lower boundary, confining the difference to $1-\elow<r<1-\elow c_{-,b,t}$, and $\hat A=0$ is immediate.

 \paragraph{Remark: local mechanism and its empirical signature.}
The mechanism of \satr{} operates locally at the surrogate level by modulating the gradient contributions of individual positions within each batch. Rather than shifting or lowering the average logged mismatch band, its predicted empirical signature is tail suppression, which is characterized by fewer runaway outward updates and the absence of late-stage collapse. By selectively adjusting these extreme deviations, \satr{} ensures that the optimization curves safely track their base recipes without undergoing catastrophic policy degradation. The comprehensive theoretical foundation of this mechanism, including all formal proofs and empirical monitoring quantities, is fully detailed in Appendix~\ref{app:monitor}.
\section{Experiments}
\label{sec:exp}

\subsection{Setup}
\label{sec:exp-setup}

\paragraph{Pipeline and model.}
All experiments run on the slime batch-wise asynchronous RL framework~\citep{zhu2025slime} with SGLang~\citep{zheng2023sglang} as the rollout engine and Megatron~\citep{shoeybi2019megatron} as the trainer, on fully decoupled GPU pools (Figure~\ref{fig:pipeline}). The backbone is Qwen3-30B-A3B-Base~\citep{yang2025qwen3}, an MoE model chosen deliberately: it activates the policy-update \emph{and} implementation mismatch terms of Equation~\eqref{eq:split}, including MoE routing. The configured lag is controlled at $n\in\{1,8\}$ trainer steps per weight broadcast, so the experimental comparisons in the main text are always within this batch-wise async regime rather than across synchronization regimes.

\paragraph{Training data, batching, and sampling.}
Training prompts total $30{,}712$ by combining {DAPO-Math-17k}~\citep{yu2025dapoopensourcellmreinforcement} and {Dolci-RL-Zero-Math-7B}~\citep{olmo2026olmo3}  and shuffled rollout. Each iteration consumes $256$ prompts with $16$ samples each, yielding $4096$ responses for a GRPO group size of $16$ across $544$ total rollout iterations. Rollouts are sampled at temperature $0.95$, top-$p$ of $1.0$, and top-$k$ of $-1$ to retain full vocabulary support as required by Lemma~\ref{lem:tv-ratio} under a $32$k-token response limit. The reward is based on the rule-based verifier. And clip radii are symmetric at $0.2$ with a constant learning rate of $10^{-6}$. \satr{} uses the hill kernel with exponent $2$ with reference quantile $0.90$. Appendix~\ref{app:expdetails} lists full hyper-parameters, hardware, and parallelism, while metrics and statistical scope are deferred to Appendix~\ref{app:metrics-scope}.

\subsection{Main-result}
\label{sec:exp-main}

\begin{table*}[t]
\centering
\caption{\textbf{Main evaluation results.} This table presents the main performance eval results and the last-epoch sampled mismatch dpi, where parenthesized values mark the best checkpoint steps. The table also notes the deep red entries that later experience collapse, with  \textbf{bold} and \underline{underline} mark the best and runner-up within each metric column.
}
\label{tab:main-combined}
\small
\setlength{\tabcolsep}{5pt}
\begin{tabular}{@{}lcccc@{}}
\toprule
Method & lag $=1$ & lag $=8$ & lag $=1$ & lag $=8$ \\
- & \multicolumn{2}{c}{AIME24 $\uparrow$} & \multicolumn{2}{c}{$\dpi\downarrow$} \\
\midrule
Qwen3-30B-A3B-Base & 9.38 & 9.38 & --- & --- \\
GRPO~\citep{shao2024deepseekmath} & 31.25 (429) & \textcolor{cCollapse}{30.17 (429 collapse)} & 0.0103 & 0.0110 \\
GSPO~\citep{zheng2025gspo} & 32.25 (474) & \textcolor{cCollapse}{31.46 (424 collapse)} & 0.0097 & 0.0109 \\
GRPO w/ R3 & 31.83 (314) & 29.58 (214) & 0.0058 & 0.0218 \\
GSPO w/ R3 & 34.00 (379) & \underline{33.13} (214) & 0.0091 & 0.0158 \\
GRPO w/ TIS & 31.62 (254) & 31.46 (374) & 0.0108 & 0.0169 \\
GSPO w/ TIS & 32.88 (319) & 31.21 (279) & 0.0151 & 0.0208 \\
GRPO w/ TIS w/ R3 & 31.79 (324) & 31.79 (239) & \textbf{0.0048} & 0.0081 \\
GSPO w/ TIS w/ R3 & 31.88 (329) & 32.92 (359) & \underline{0.0054} & \textbf{0.0068} \\
DPPO~\citep{qi2026dppo} & 33.96 (324) & 32.71 (289) & 0.0108 & 0.0118 \\
\rowcolor{metabg!8}
\textbf{\satr-GRPO} & 32.71 (404) & 32.08 (299) & 0.0101 & 0.0108 \\
\rowcolor{metabg!8}
\textbf{\satr-GSPO} & \underline{34.17} (254) & 32.71 (359) & 0.0095 & 0.0107 \\
\rowcolor{metabg!8}
\textbf{\satr-GRPO w/ R3} & 33.75 (284) & \underline{33.13} (289) & 0.0056 & 0.0079 \\
\rowcolor{metabg!8}
\textbf{\satr-GSPO w/ R3} & \textbf{35.83} (379) & \textbf{34.79} (359) & 0.0056 & \underline{0.0076} \\
\bottomrule
\end{tabular}
\end{table*}

\paragraph{Baselines.}
We evaluate combinations of \textbf{GRPO}~\citep{shao2024deepseekmath} and \textbf{GSPO}~\citep{zheng2025gspo} across \textbf{vanilla}, \textbf{R3}~\citep{ma2025r3}, \textbf{TIS}~\citep{zheng2025tis}, and \textbf{combined stabilizer} configurations. This grid is augmented by four \satr variants and \textbf{DPPO}~\citep{qi2026dppo}, specifically the variant using $D_{\mathrm{TV}}$ to denote total variation divergence. Appendix~\ref{app:stabilizers} unifies these methods within one design space. All configurations share identical optimization, data, and rewards, varying only by stabilizer and configured lag. Appendix~\ref{app:expdetails} tabulates all reproducible hyperparameters.

\textbf{\satr achieves better performance in high staleness scenerio.} Table~\ref{tab:main-combined} shows that \textbf{\satr-GSPO w/ R3 ranks first at both lag settings, reaching $35.83$ at lag $1$ and $34.79$ at lag $8$}. The lead persists as the configured lag increases, indicating that the advantage is not confined to the milder asynchronous regime. Relative to GRPO and GSPO, the gains are $4.58$ and $3.58$ points at lag $1$, and $4.62$ and $3.33$ points at lag $8$; relative to DPPO, the margins are $1.87$ and $2.08$ points. Overall, the combination of \satr{}, GSPO, and R3 is the strongest configuration in the reported grid, including the more demanding lag-$8$ setting.

\textbf{\satr{} variant with top performance also maintains comparatively low observed staleness.} The $\dpi\downarrow$ columns show that sampled mismatch rises with configured lag for all methods, but the better-performing non-collapsed configurations occupy a lower mismatch band than the plain baselines. In particular, \textbf{\satr-GSPO w/ R3 records $\dpi=0.0056$ and $0.0076$, compared with $0.0097$ and $0.0109$ for GSPO}. Furthermore, we observe that under the vanilla configuration, both GRPO and GSPO exhibit training collapse respectively at 429 and 424 steps in high-staleness regimes. This highlights the severe risk of accelerating throughput, demonstrating its critical threat to training stability within reinforcement learning frameworks.

\subsection{Training--inference mismatch trajectories over training}
\label{sec:exp-drift}

\begin{figure}[tbp]
\centering
\includegraphics[width=0.85\linewidth]{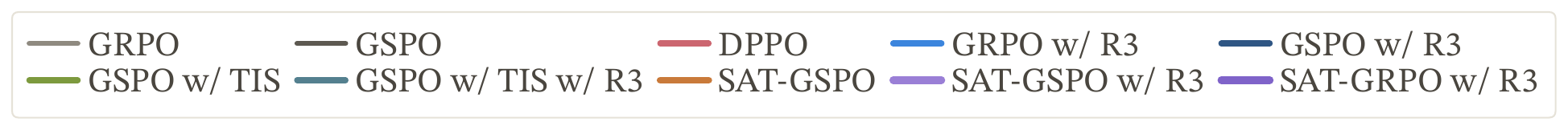}
\vspace{-0.2em}

\begin{subfigure}[t]{0.495\linewidth}
\centering
\includegraphics[width=\linewidth]{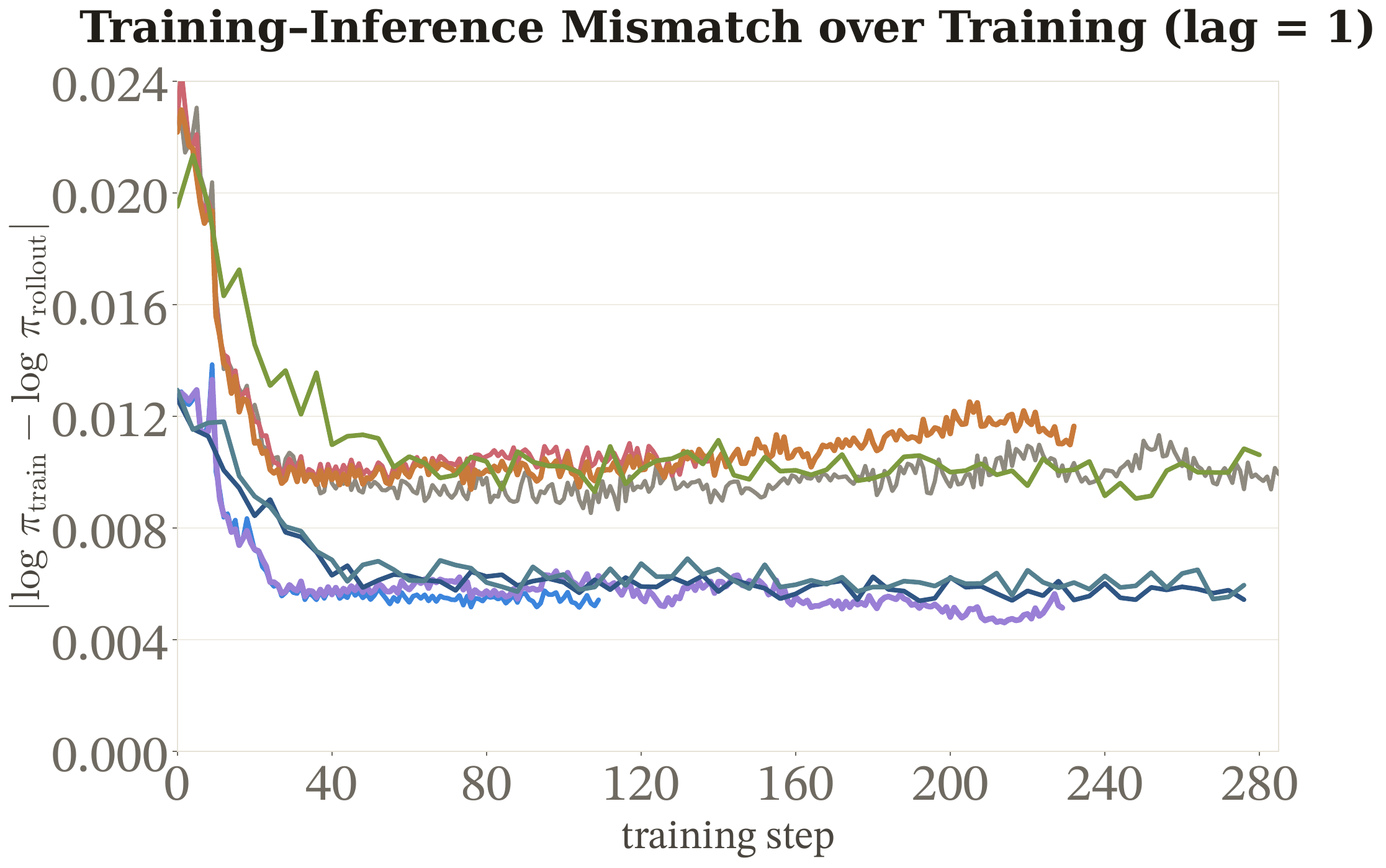}
\caption{configured lag $=1$}
\label{fig:drift-s1}
\end{subfigure}\hfill
\begin{subfigure}[t]{0.495\linewidth}
\centering
\includegraphics[width=\linewidth]{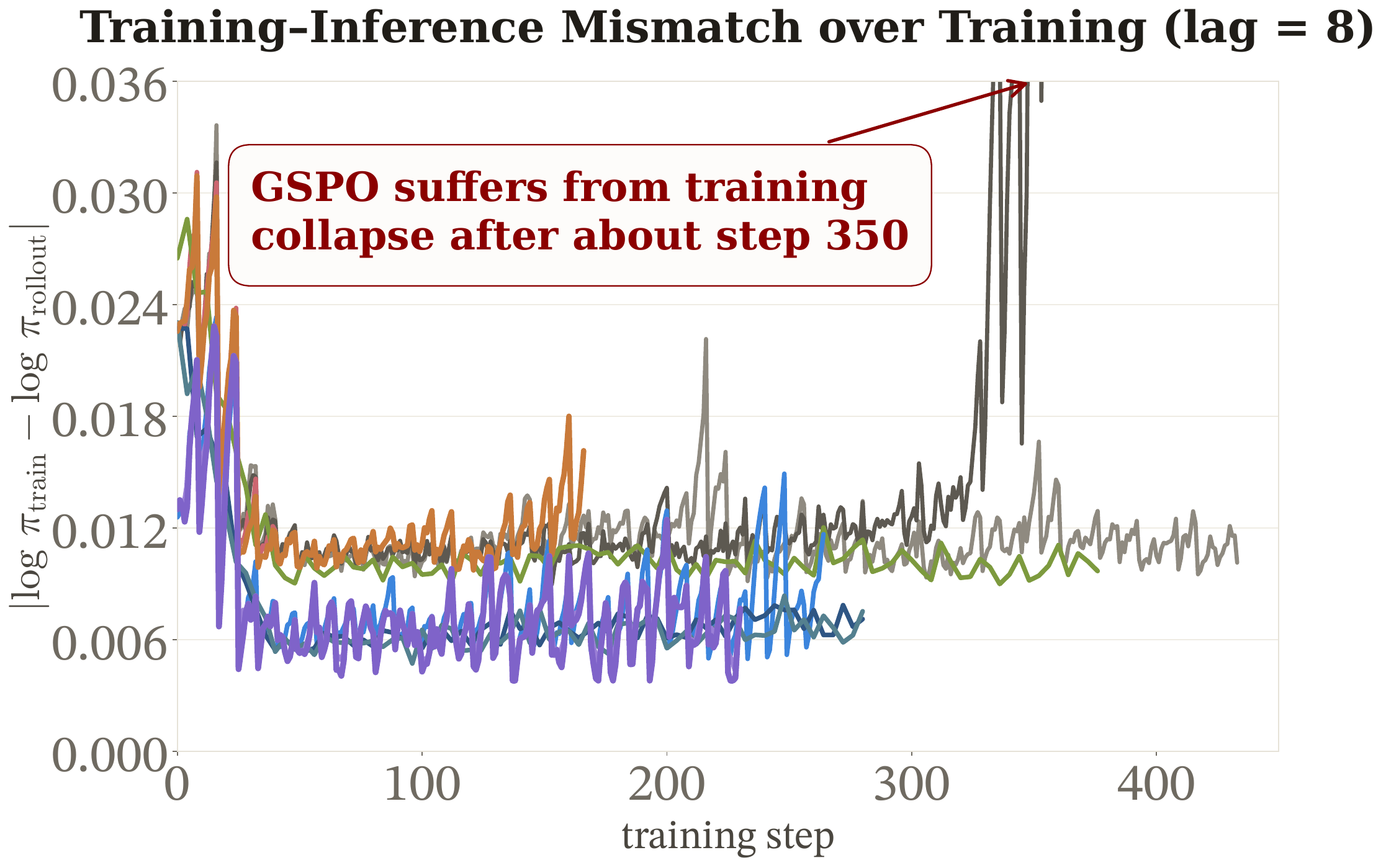}
\caption{configured lag $=8$}
\label{fig:drift-s8}
\end{subfigure}
\vspace{0.3em}
\caption{\textbf{Training--inference mismatch $\dpi$ over training at configured lag $1$ and $8$.} This figure tracks the batch-averaged sampled mismatch of variance methods across training, with a shared legend distinguishing configurations and separate panels contrasting the mild-lag setting with the high-lag setting.}
\label{fig:drift-pair}
\end{figure}

\textbf{Routing replay stabilizes the log prob difference while \satr{} thrives at lower values.} Figure~\ref{fig:drift-pair} tracks the mismatch diagnostic across variance methods behind Table~\ref{tab:main-combined}. 
The integration of routing routing replay (R3) is essential for stabilization: methods utilizing R3 restrict the log prob difference to a tight fluctuation around 0.007. 
Conversely, without R3, the metric deviates noticeably, drifting to around 0.011. 
Second, in the later stages of GRPO training, omitting R3 results in a sudden, sharp growth in the log prob difference. 
Moreover, \textbf{\satr{} achieves better performance while maintaining a lower log prob difference} without training collapse, proving its robustness in managing training dynamics without relying on excessive policy divergence.

\subsection{Asymmetric clipping in \satr{}}
\label{sec:exp-asymclip}

The benchmark table and the mismatch trajectories suggest that the stabilizers differ not only in strength but also in update geometry. A useful organizing principle is the asymmetric clipping family, which includes PPO, DPPO, SPO, DRPO, and \satr{}. For a fixed sampled token with ratio $r_t$ and advantage $\hat A_t$, the unclipped surrogate has derivative $\hat A_t$ with respect to $r_t$, so the sign of $\hat A_t$ determines the locally preferred direction. Relative to $r_t=1$, the sampled update is locally diverging when $\operatorname{sign}(\hat A_t(r_t-1))>0$ and locally converging when $\operatorname{sign}(\hat A_t(r_t-1))<0$. PPO therefore keeps the gate:
\begin{equation}
M_t^{\mathrm{PPO}}
=\mathbf 1\Bigl\{\operatorname{sign}(\hat A_t(r_t-1))\le 0\ \vee\ |r_t-1|\le\varepsilon\Bigr\},
\label{eq:exp-asym-ppo-mask}
\end{equation}
which clips only those updates that would move the sampled ratio farther away from one and preserves the updates that move it back toward the behavior reference.


DPPO retains the directional logic but updates the discriminator. The sampled total variation distance $D_t^{\mathrm{TV}}$ bounds the ratio deviation $|r_t-1|$, logically defining the induced gate $M_t^{\mathrm{DPPO}}$:
\begin{equation}
\begin{aligned}
D_t^{\mathrm{TV}} &=\bigl|\pi^{(j)}(y_t\mid s_t)-\mu(y_t\mid s_t)\bigr| =\mu(y_t\mid s_t)\,|r_t-1|, \\
|r_t-1| &\le \frac{\delta}{\mu(y_t\mid s_t)}, \\
M_t^{\mathrm{DPPO}} &=\mathbf 1\Bigl\{\operatorname{sign}(\hat A_t(r_t-1))\le 0\ \vee\ |r_t-1|\le\delta/\mu(y_t\mid s_t)\Bigr\}.
\end{aligned}
\label{eq:dppo_formulation}
\end{equation}

This gating mechanism imposes a more stringent constraint on high-probability head tokens, while granting a more permissive update allowance to rare, low-probability tokens residing in the long tail of the distribution.


 The results in Table~\ref{tab:main-combined} and Figure~\ref{fig:asym-dppo-bar} support this view. DPPO reaches $33.96$ at lag $1$ and $32.71$ at lag $8$, outperforming the GRPO baseline at both lags and avoiding the later collapse seen in GRPO and GSPO. This shows that an asymmetric gate with a token-adaptive $D_{\mathrm{TV}}$ threshold can already improve stability in the reported async regime, although DPPO still trails \satr-GSPO at lag $1$ and \satr-GSPO w/ R3 at both lags.

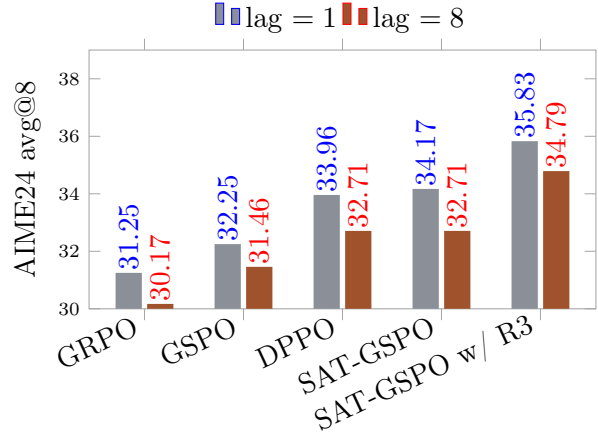
\begin{wrapfigure}{r}{0.5\textwidth}
\vspace{-1.5em}
\centering
\begin{tikzpicture}
\begin{axis}[
width=0.5\textwidth,
height=5cm,
ybar=2pt,
bar width=9.9pt,
ymin=30,
ymax=39,
symbolic x coords={GRPO,GSPO,DPPO,{SAT-GSPO},{SAT-GSPO w/ R3}},
xtick=data,
xticklabel style={font=\fontsize{11.25}{13.5}\selectfont,rotate=25,anchor=east},
yticklabel style={font=\fontsize{7.5}{9}\selectfont},
label style={font=\fontsize{10.5}{12}\selectfont},
legend style={font=\fontsize{11.25}{13.5}\selectfont,draw=none,fill=none,at={(0.5,1.03)},anchor=south,legend columns=2},
ylabel={AIME24 avg@8},
enlarge x limits=0.14,
axis line style={black!35},
tick style={black!35},
nodes near coords,
nodes near coords style={
    font=\fontsize{11.25}{13.5}\selectfont,
    rotate=90,
    anchor=west,
    inner sep=1pt,
    /pgf/number format/fixed,
    /pgf/number format/fixed zerofill,
    /pgf/number format/precision=2
},
every node near coord/.append style={xshift=0pt,yshift=0pt}
]
\addplot+[draw=none,fill=cGRPO] coordinates {(GRPO,31.25) (GSPO,32.25) (DPPO,33.96) ({SAT-GSPO},34.17) ({SAT-GSPO w/ R3},35.83)};
\addplot+[draw=none,fill=cSienna] coordinates {(GRPO,30.17) (GSPO,31.46) (DPPO,32.71) ({SAT-GSPO},32.71) ({SAT-GSPO w/ R3},34.79)};
\legend{lag $=1$, lag $=8$}
\end{axis}
\end{tikzpicture}
\caption{AIME24 avg@8 for GRPO, GSPO, DPPO, and \satr{} variants.}
\label{fig:asym-dppo-bar}
\vspace{-3.5em}
\end{wrapfigure}

Together with Figure~\ref{fig:drift-pair}, the comparison also shows that asymmetry alone is not sufficient: the rule used to set the effective boundary matters. DPPO and \satr{} both preserve pull-back updates and intercept outward ones, but \satr{} chooses the active boundary from the observed mismatch tail of the current batch. Equations~\eqref{eq:satr-clip} and~\eqref{eq:satrloss}, together with Propositions~\ref{prop:containment} and~\ref{prop:pessimism}, formalize this point: \satr{} maintains the asymmetric attribute while cutting off only the newly intercepted outward band. A broader method-by-method comparison is deferred to Appendix~\ref{app:stabilizers} and Appendices~\ref{app:gspo-seqclip}--\ref{app:rollout-lp}.

\FloatBarrier

\section{Open Questions}
\label{sec:limitations}

While the current experimental scope only focusing the setting of reinforcement learning for improving LLM's reasoning capability, the underlying SAT methodology demonstrates strong potential for broader application across diverse learning domains, leaving main limitations as follows:
\begin{itemize}
\item \textbf{Surrogate-level control.} \satr{} contracts a sampled nominal interval rather than the realized policy itself. The learned policy can still move outside that interval, unobserved vocabulary actions remain unconstrained, and the proxy $|\log r|$ should not be identified with either $D_{\mathrm{TV}}$ or version age because it omits the probability-mass weight in Equation~\eqref{eq:tv-mass} and mixes policy lag with implementation mismatch.
\item \textbf{Relative adaptation.} Because the reference quantile $q$ is recomputed for every batch, the method responds to the current mismatch tail rather than to an absolute scale. It therefore does not imply monotone shrinkage with configured lag once PPO's original boundary is already active.
\end{itemize}

\section{Conclusion}
\label{sec:conclusion}

The trust-region interpretation of \satr{} is narrow but clear. $D_{\mathrm{TV}}$ controls a finite-horizon approximation term, and $D_{\mathrm{TV}}$ is exactly a behavior expectation of absolute ratio deviation, so a hard all-action ratio bound would control $D_{\mathrm{TV}}$. PPO does not impose such a bound and instead clips an advantage-dependent sampled surrogate. \satr{} operates at that same surrogate level by using a detached mismatch proxy to contract the sign-selected nominal radius on the high-staleness tail of each batch. Its effect is therefore local but precise: it cuts off newly intercepted outward updates earlier while preserving pull-back gradients, and it reduces exactly to the underlying recipe when disabled. On the reported grid, \satr-GSPO w/ R3 attains the best-observed AIME24 avg@8 at both configured lags, reaches $35.83$ at lag $1$ and $34.79$ at lag $8$, and avoids the late collapse seen in several fixed-clip lag-$8$ baselines. Together with the interval-containment and pointwise-pessimism results, these findings give a concrete account of what an adaptive $\varepsilon$ can contribute in asynchronous RL.

\bibliographystyle{plainnat}
\bibliography{refs}

\newpage
\appendix

\section{Notation}
\label{app:notation}

\begin{table}[h]
\centering
\caption{Notation used throughout the paper.}
\label{tab:notation}
\small
\setlength{\tabcolsep}{6pt}
\begin{tabular}{@{}ll@{}}
\toprule
Symbol & Meaning \\
\midrule
$x,\ y_b,\ s_{b,t}$ & Prompt; response $y_b=(y_{b,1},\dots,y_{b,T_b})$; state $s_{b,t}=(x,y_{b,1},\dots,y_{b,t-1})$. \\
$b,\ t,\ j$ & Response index; token index; absolute trainer-update index. \\
$T_b,\ \mathcal T_{\mathrm{act}}$ & Length of response $b$; active (non-padding) token positions of the microbatch. \\
$\mu_{b,t}$ & Behavior policy, namely the distribution the rollout engine actually sampled token $(b,t)$ from. \\
$\pi^{(j)}$ & Current trainer policy at update $j$. \\
$\pi^{(\ell)}$ & Trainer-side reference log-probability at the rollout checkpoint. \\
$n,\ N_{b,t}$ & Configured pipeline lag (trainer steps per weight broadcast); realized per-token lag. \\
$r_{b,t},\ d_{b,t}$ & Sampled-token importance ratio and its log-ratio, $d_{b,t}=\log r_{b,t}$. \\
$\rho^{\mathrm{seq}}_b$ & GSPO length-normalized sequence ratio in Equation~\eqref{eq:gspo-ratio}. \\
$\hat A_{b,t},\ R(y),\ \xi$ & Advantage estimate; sequence reward; reward bound with $|R(y)|\le\xi$. \\
$\dtv$ & Per-state total variation $\tfrac12\sum_a|\mu(a\mid s)-\pi(a\mid s)|$. \\
$L'_{\mu}(\pi)$ & Finite-horizon linear surrogate in Equation~\eqref{eq:surrogate}. \\
$\elow,\ \ehigh$ & Base clip radii; $\tilde\varepsilon_{\mathrm{low}/\mathrm{high},b,t}$ are the \satr{} radii in Equation~\eqref{eq:adaptive-eps}. \\
$\psi(u;q),\ q,\ c_{\pm,b,t}$ & Hill kernel; detached microbatch quantile reference; gated contraction factors. \\
$\dpi$ & Sampled training--inference mismatch $\E_{(b,t)\in\mathcal T_{\mathrm{act}}}|\log r^{\mathrm{tok}}_{b,t}|$ from Equation~\eqref{eq:mismatch-metric}. \\
\bottomrule
\end{tabular}
\end{table}

\section{Synchronization Regimes of Asynchronous RL}
\label{app:regimes}

The main text now keeps only the batch-wise asynchronous pipeline in Figure~\ref{fig:pipeline}, because all reported experiments are conducted in that regime. This appendix records the two endpoint cases referenced there---Appendix~\ref{app:sync-rl} for synchronous RL and Appendix~\ref{app:fully-async} for fully asynchronous RL. They clarify why the configured lag is only a coarse systems label and why the realized mismatch can vary substantially even when the nominal training recipe stays fixed.

\subsection{Synchronous RL}
\label{app:sync-rl}

In synchronous RL, rollout and optimization alternate on one shared GPU pool and every batch is trained by the same parameter vector that produced it:
\begin{equation}
j=\ell,
\qquad
\theta_{\mathrm{roll}}=\theta_{\mathrm{train}}=\theta^{(\ell)},
\qquad
N_{b,t}=0,
\qquad
\mu_{b,t}=\pi^{(\ell)},
\qquad
r_{b,t}=1\ \text{a.s.}
\label{eq:sync-identity}
\end{equation}
At this pre-update behavior point, the sampled ratio carries no off-policy correction, and under the usual on-policy assumptions the surrogate gradient coincides with the corresponding policy-gradient estimator. This regime therefore removes staleness as a learning issue, but it pays a clear systems cost: generation, optimization, and weight broadcast are serialized, so the two engines never overlap.

\begin{figure}[h]
\centering
\includegraphics[width=0.88\linewidth]{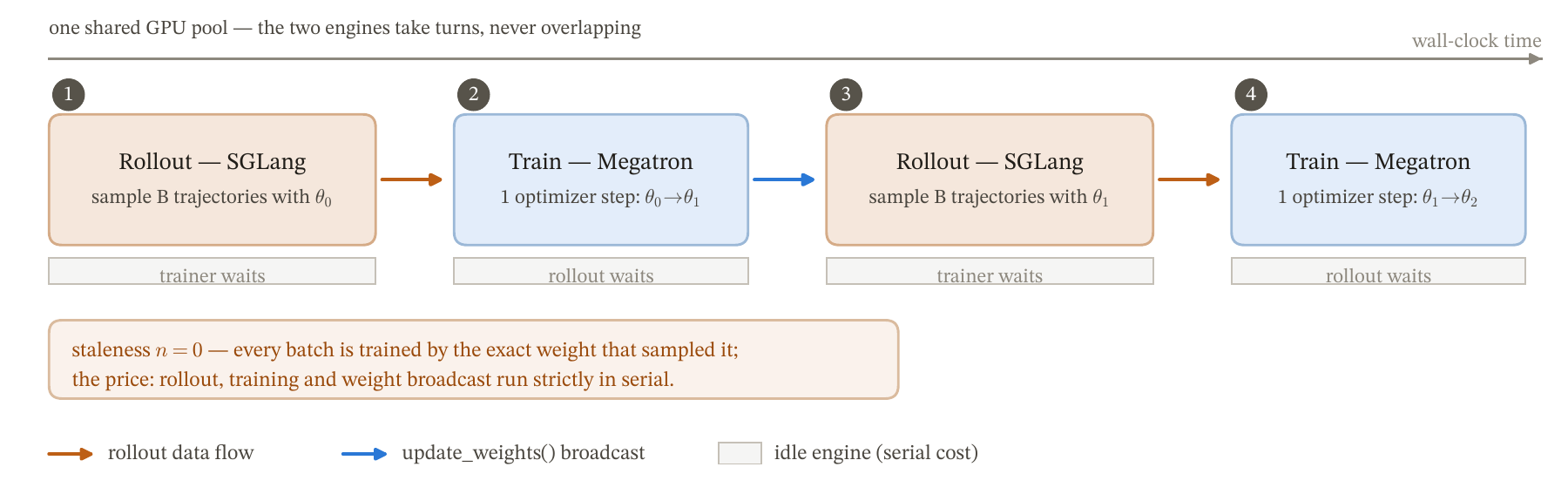}
\caption{Synchronous RL. Rollout and training share one GPU pool and alternate in time, so each batch is consumed by the same weight version that generated it.}
\label{fig:pipeline-sync-app}
\end{figure}

\subsection{Fully asynchronous RL}
\label{app:fully-async}

Fully asynchronous RL removes the fixed batchwise handoff. Rollouts stream continuously into a queue, the trainer pulls data whenever the queue has enough completed trajectories, and new weights can be hot-swapped into the rollout engine while a response is still being decoded. A single trajectory can therefore span several policy versions before it is even enqueued, and then wait under still newer versions before it is trained. A convenient decomposition writes the realized lag as:
\begin{equation}
N_{b,t}=N_{\mathrm{PQS},b,t}+N_{\mathrm{IQS},b,t},
\label{eq:fully-async-lag}
\end{equation}
where $N_{\mathrm{PQS},b,t}$ counts the versions that elapse while the trajectory is being generated and $N_{\mathrm{IQS},b,t}$ counts the versions that elapse while the completed trajectory waits in the queue. Following \citet{dong2026staleness}, the first term grows with the response-length tail of the workload and the second depends on queue occupancy and the rollout-to-trainer throughput ratio.

This regime is important conceptually because it makes the central quantity of the paper random at the token level. Even if the system exposes one nominal queue depth or one nominal update cadence, the relevant quantity for optimization is the realized sampled mismatch $d_{b,t}=\log r_{b,t}$, not a single preconfigured scalar. This is precisely the setting in which a batch-adaptive mechanism such as \satr{} becomes more natural than a rule keyed to a fixed lag alone.

\begin{figure}[h]
\centering
\includegraphics[width=0.88\linewidth]{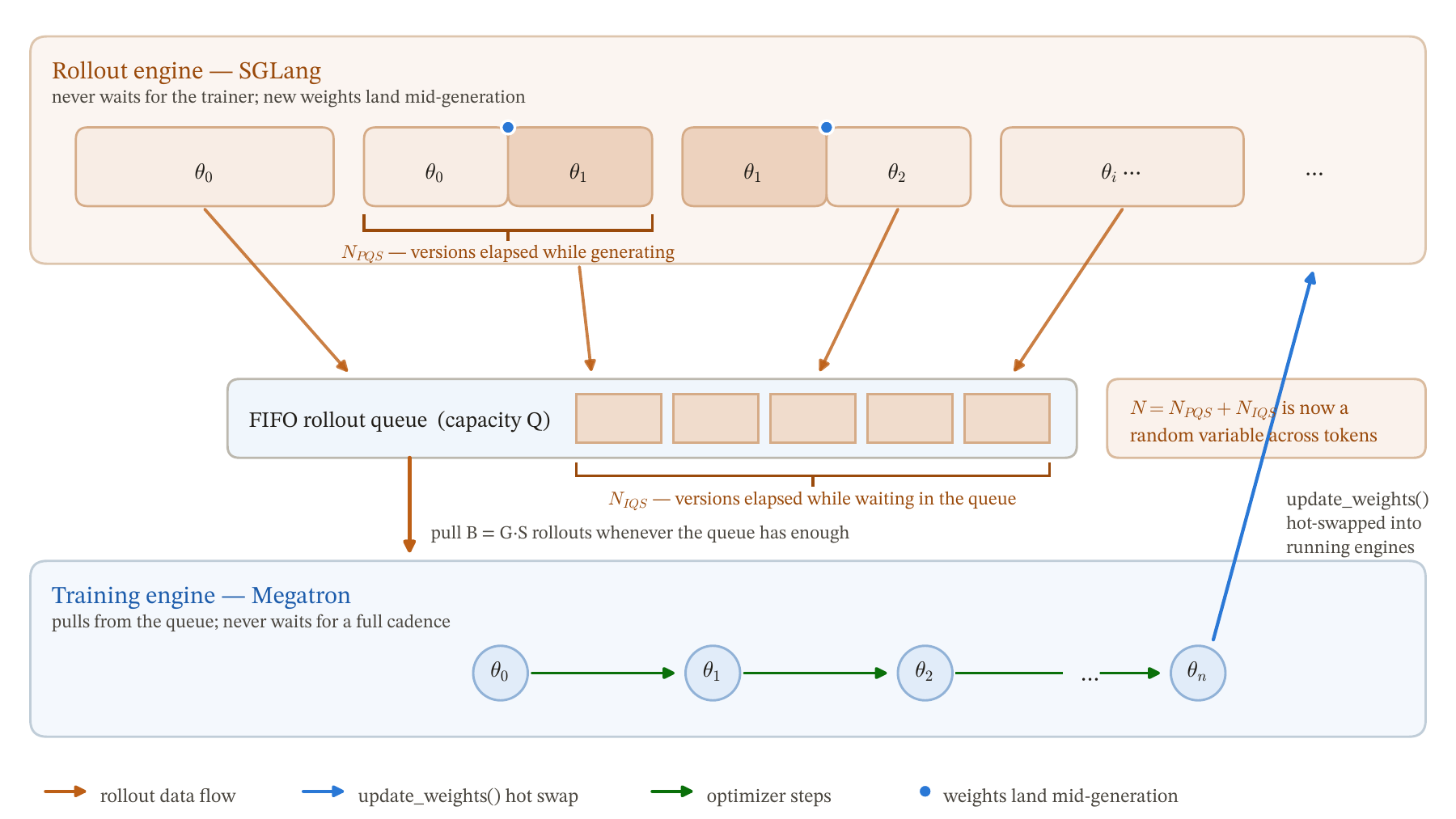}
\caption{Fully asynchronous RL. Rollouts stream through a queue, weights can change during generation, and the realized lag varies across tokens.}
\label{fig:pipeline-fully-app}
\end{figure}

\section{Mismatch Sources as Diagnostic Hypotheses}
\label{app:sources}

Equation~\eqref{eq:split} in the main text gives an exact two-term split of the sampled log-ratio. Write $\Delta_{b,t}^{\pi}=\log\pi^{(j)}-\log\pi^{(\ell)}$ for policy-update mismatch and $\Delta_{b,t}^{\mathrm{impl}}=\log\pi^{(\ell)}-\log\mu_{b,t}$ for implementation mismatch, both evaluated at $(y_{b,t},s_{b,t})$. Then:
\begin{equation}
d_{b,t}=\log r_{b,t}=\Delta_{b,t}^{\pi}+\Delta_{b,t}^{\mathrm{impl}}.
\label{eq:app-split}
\end{equation}
The term $\Delta_{b,t}^{\pi}$ records policy-update mismatch inside the asynchronous window. The term $\Delta_{b,t}^{\mathrm{impl}}$ records implementation mismatch between the trainer-side reference and the actual rollout engine. The purpose of the present section is not to produce a unique additive decomposition of all sources of mismatch, because such a decomposition would require additional intermediate distributions that the experiments do not log. Instead, the section records the hypotheses that are operationally useful when interpreting the measured $d_{b,t}$.

Policy updates are the most direct source. If a rollout is produced under $\theta^{(\ell)}$ and consumed after several optimizer steps, then a first-order expansion along the optimization path gives:
\begin{equation}
\log\pi^{(j)}(y_{b,t}\mid s_{b,t})-\log\pi^{(\ell)}(y_{b,t}\mid s_{b,t})
\;\approx\;\sum_{k=\ell}^{j-1}\eta\,\big\langle\nabla_\theta\log\pi(y_{b,t}\mid s_{b,t}),\,g^{(k)}\big\rangle,
\label{eq:taylor-app}
\end{equation}
with learning rate $\eta$ and update directions $g^{(k)}$. This term is exactly the object targeted by PPO-style importance reweighting. It need not be centered at zero, because consecutive updates tend to move in reward-improving directions rather than cancel each other.

Implementation mismatch remains even when the nominal weights are identical. Rollout and training may differ in fused kernels, attention kernels, KV-cache layout, numerical precision, or mixture-of-experts routing. The router case is particularly delicate because the Top-$K$ mask is a discontinuous function of the router logits. Small numerical differences can activate a different expert set and change the forward pass abruptly. Hardware differences can amplify the same effect: the same model and code path can produce train--rollout KL values that differ by orders of magnitude across GPU generations. The paper therefore treats these mechanisms as diagnostic hypotheses that explain observed mismatch and motivate stabilizers such as R3, rather than as terms of an exact decomposition used in the proof.

\section{Derivation of the Finite-Horizon Improvement Bound}
\label{app:improvement-bound}

This section collects the full derivation behind the finite-horizon trust-region discussion in Section~\ref{sec:tv}. We keep the notation of the main paper. A sequence $y=(y_1,\dots,y_T)$ is generated autoregressively from prompt $x$, the behavior policy is $\mu$, the target policy is $\pi$, rewards satisfy $|R(y)|\le\xi$, and the finite-horizon objective is $\mathcal J(\pi)=\mathbb E_{y\sim\pi}[R(y)]$. The surrogate used in the main text is:
\begin{equation}
L'_{\mu}(\pi)=\mathbb E_{y\sim\mu}\left[R(y)\sum_{t=1}^{T}\left(\frac{\pi(y_t\mid s_t)}{\mu(y_t\mid s_t)}-1\right)\right].
\label{eq:app-surrogate}
\end{equation}

\subsection{An exact performance-difference identity}
\label{app:exact-identity}

\begin{lemma}[Exact finite-horizon identity]
\label{lem:app-identity}
For any two sequence policies $\mu$ and $\pi$ with compatible support on the $\mu$-reachable states,
\begin{equation}
\mathcal J(\pi)-\mathcal J(\mu)
= L'_{\mu}(\pi)-\Delta(\mu,\pi),
\label{eq:app-identity}
\end{equation}
where
\begin{align}
L'_{\mu}(\pi)
&= \mathbb E_{y\sim\mu}\left[R(y)\sum_{t=1}^{T}\left(\frac{\pi(y_t\mid s_t)}{\mu(y_t\mid s_t)}-1\right)\right],
\label{eq:app-identity-surrogate}\\
\Delta(\mu,\pi)
&= \mathbb E_{y\sim\mu}\left[R(y)\sum_{t=1}^{T}\left(\frac{\pi(y_t\mid s_t)}{\mu(y_t\mid s_t)}-1\right)\left(1-\prod_{j=t+1}^{T}\frac{\pi(y_j\mid s_j)}{\mu(y_j\mid s_j)}\right)\right].
\label{eq:app-delta}
\end{align}
\end{lemma}

\textbf{Proof.}
Start from the definition of the return difference:
\begin{align*}
\mathcal J(\pi)-\mathcal J(\mu)
&=\mathbb E_{y\sim\pi}[R(y)]-\mathbb E_{y\sim\mu}[R(y)]\\
&=\sum_{y}\bigl(\pi(y\mid x)-\mu(y\mid x)\bigr)R(y).
\end{align*}
The sequence-probability difference admits the telescoping expansion:
\begin{equation*}
\pi(y\mid x)-\mu(y\mid x)
=\sum_{t=1}^{T}\left(\prod_{k=1}^{t-1}\mu(y_k\mid s_k)\right)
\bigl(\pi(y_t\mid s_t)-\mu(y_t\mid s_t)\bigr)
\left(\prod_{j=t+1}^{T}\pi(y_j\mid s_j)\right).
\end{equation*}
Substituting this identity and factoring out $\mu(y\mid x)$ gives:
\begin{align*}
\mathcal J(\pi)-\mathcal J(\mu)
&=\sum_{y}\mu(y\mid x)R(y)
\sum_{t=1}^{T}\left(\frac{\pi(y_t\mid s_t)}{\mu(y_t\mid s_t)}-1\right)
\left(\prod_{j=t+1}^{T}\frac{\pi(y_j\mid s_j)}{\mu(y_j\mid s_j)}\right)\\
&=\mathbb E_{y\sim\mu}\left[R(y)
\sum_{t=1}^{T}\left(\frac{\pi(y_t\mid s_t)}{\mu(y_t\mid s_t)}-1\right)
\left(\prod_{j=t+1}^{T}\frac{\pi(y_j\mid s_j)}{\mu(y_j\mid s_j)}\right)
\right].
\end{align*}
Adding and subtracting the term with the future-product replaced by $1$ yields:
\begin{align*}
\mathcal J(\pi)-\mathcal J(\mu)
&=\mathbb E_{y\sim\mu}\left[R(y)\sum_{t=1}^{T}\left(\frac{\pi(y_t\mid s_t)}{\mu(y_t\mid s_t)}-1\right)\right]\\
&\quad-\mathbb E_{y\sim\mu}\left[R(y)\sum_{t=1}^{T}\left(\frac{\pi(y_t\mid s_t)}{\mu(y_t\mid s_t)}-1\right)
\left(1-\prod_{j=t+1}^{T}\frac{\pi(y_j\mid s_j)}{\mu(y_j\mid s_j)}\right)\right],
\end{align*}
which is exactly Equations~\eqref{eq:app-identity-surrogate} and~\eqref{eq:app-delta}.

\subsection{A sequence-level total-variation bound}
\label{app:sequence-tv}

\begin{lemma}[Sequence $D_{\mathrm{TV}}$ is bounded by cumulative one-step $D_{\mathrm{TV}}$]
\label{lem:app-sequence-tv}
Let $\mu_T(\cdot\mid s_1)$ and $\pi_T(\cdot\mid s_1)$ denote the induced sequence distributions over responses of length $T$. Then
\begin{equation}
D_{\mathrm{TV}}\bigl(\mu_T(\cdot\mid s_1)\,\|\,\pi_T(\cdot\mid s_1)\bigr)
\le
\sum_{t=1}^{T}\mathbb E_{s_t\sim\mu}\Big[D_{\mathrm{TV}}\bigl(\mu(\cdot\mid s_t)\,\|\,\pi(\cdot\mid s_t)\bigr)\Big].
\label{eq:app-sequence-tv}
\end{equation}
\end{lemma}

\textbf{Proof.}
Write $P(y)=\mu_T(y\mid s_1)$ and $Q(y)=\pi_T(y\mid s_1)$. Then
\begin{equation*}
2D_{\mathrm{TV}}(P\|Q)=\sum_{y}|P(y)-Q(y)|.
\end{equation*}
Using the product-difference identity:
\begin{equation*}
a_1\cdots a_T-b_1\cdots b_T
=\sum_{t=1}^{T}\left(\prod_{k=1}^{t-1}a_k\right)(a_t-b_t)\left(\prod_{j=t+1}^{T}b_j\right)
\end{equation*}
and the triangle inequality,
\begin{align*}
2D_{\mathrm{TV}}(P\|Q)
&\le \sum_{t=1}^{T}\sum_{y}
\left(\prod_{k=1}^{t-1}\mu(y_k\mid s_k)\right)
\bigl|\mu(y_t\mid s_t)-\pi(y_t\mid s_t)\bigr|
\left(\prod_{j=t+1}^{T}\pi(y_j\mid s_j)\right).
\end{align*}
For each fixed $t$, summing over $y_{t+1},\dots,y_T$ collapses the trailing product to $1$, which leaves:
\begin{align*}
2D_{\mathrm{TV}}(P\|Q)
&\le \sum_{t=1}^{T}\sum_{y_1,\dots,y_{t-1}}
\left(\prod_{k=1}^{t-1}\mu(y_k\mid s_k)\right)
\sum_{y_t}\bigl|\mu(y_t\mid s_t)-\pi(y_t\mid s_t)\bigr|.
\end{align*}
The inner sum is $2D_{\mathrm{TV}}(\mu(\cdot\mid s_t)\|\pi(\cdot\mid s_t))$, while the outer sum is the expectation over $s_t$ induced by $\mu$. Dividing by $2$ proves the claim.

\subsection{Max-divergence and average-divergence forms}
\label{app:bound-proofs}

\begin{proposition}[Max-divergence and average-divergence bounds]
\label{prop:app-bounds}
Under the assumptions above,
\begin{equation}
\mathcal J(\pi)-\mathcal J(\mu)
\ge L'_{\mu}(\pi)-2\xi T(T-1)\bigl[D_{\mathrm{TV}}^{\max}(\mu,\pi)\bigr]^2,
\label{eq:app-max-bound}
\end{equation}
where $D_{\mathrm{TV}}^{\max}(\mu,\pi)=\sup_s D_{\mathrm{TV}}(\mu(\cdot\mid s)\|\pi(\cdot\mid s))$, and also
\begin{equation}
\mathcal J(\pi)-\mathcal J(\mu)
\ge L'_{\mu}(\pi)-4\xi\,\mathbb E_{y\sim\mu}\left[\sum_{t=1}^{T}D_{\mathrm{TV}}\bigl(\mu(\cdot\mid s_t)\,\|\,\pi(\cdot\mid s_t)\bigr)\right].
\label{eq:app-avg-bound}
\end{equation}
Equation~\eqref{eq:app-avg-bound} is the form reported in the main text as Equation~\eqref{eq:tvbound}.
\end{proposition}

\textbf{Proof of the max-divergence form}
By Lemma~\ref{lem:app-identity},
\begin{equation*}
\mathcal J(\pi)-\mathcal J(\mu)=L'_{\mu}(\pi)-\Delta(\mu,\pi).
\end{equation*}
To bound $\Delta(\mu,\pi)$, first use $|R(y)|\le\xi$:
\begin{align}
\Delta(\mu,\pi)
&\le \xi\sum_{t=1}^{T}\mathbb E_{y_{\le t}\sim\mu}\left[
\left|\frac{\pi(y_t\mid s_t)}{\mu(y_t\mid s_t)}-1\right|
\mathbb E_{y_{>t}\sim\mu(\cdot\mid s_{t+1})}
\left|1-\frac{\pi(y_{>t}\mid s_{t+1})}{\mu(y_{>t}\mid s_{t+1})}\right|
\right].
\label{eq:app-delta-step1}
\end{align}
The inner expectation is exactly twice the $D_{\mathrm{TV}}$ divergence between the future-trajectory distributions:
\begin{equation*}
\mathbb E_{y_{>t}\sim\mu(\cdot\mid s_{t+1})}
\left|1-\frac{\pi(y_{>t}\mid s_{t+1})}{\mu(y_{>t}\mid s_{t+1})}\right|
=2D_{\mathrm{TV}}\bigl(\mu_{>t}(\cdot\mid s_{t+1})\,\|\,\pi_{>t}(\cdot\mid s_{t+1})\bigr).
\end{equation*}
Applying Lemma~\ref{lem:app-sequence-tv} to this future-horizon $D_{\mathrm{TV}}$ and then upper-bounding each term by $D_{\mathrm{TV}}^{\max}(\mu,\pi)$ gives
\begin{equation*}
D_{\mathrm{TV}}\bigl(\mu_{>t}(\cdot\mid s_{t+1})\,\|\,\pi_{>t}(\cdot\mid s_{t+1})\bigr)
\le (T-t)D_{\mathrm{TV}}^{\max}(\mu,\pi).
\end{equation*}
Substituting into Equation~\eqref{eq:app-delta-step1} yields
\begin{align*}
\Delta(\mu,\pi)
&\le 2\xi D_{\mathrm{TV}}^{\max}(\mu,\pi)
\sum_{t=1}^{T}(T-t)
\mathbb E_{s_t\sim\mu}\left[\sum_{y_t}\mu(y_t\mid s_t)
\left|\frac{\pi(y_t\mid s_t)}{\mu(y_t\mid s_t)}-1\right|\right]\\
&=2\xi D_{\mathrm{TV}}^{\max}(\mu,\pi)
\sum_{t=1}^{T}(T-t)
\mathbb E_{s_t\sim\mu}\Big[2D_{\mathrm{TV}}(\mu(\cdot\mid s_t)\|\pi(\cdot\mid s_t))\Big]\\
&\le 4\xi\bigl[D_{\mathrm{TV}}^{\max}(\mu,\pi)\bigr]^2\sum_{t=1}^{T}(T-t)\\
&=2\xi T(T-1)\bigl[D_{\mathrm{TV}}^{\max}(\mu,\pi)\bigr]^2.
\end{align*}
Substituting this bound into the exact identity proves Equation~\eqref{eq:app-max-bound}.

\textbf{Proof of the average-divergence form}
Return to Equation~\eqref{eq:app-delta-step1}. Instead of upper-bounding the future-trajectory $D_{\mathrm{TV}}$ by $(T-t)D_{\mathrm{TV}}^{\max}$, use the trivial bound $D_{\mathrm{TV}}\le 1$. This gives:
\begin{align*}
\Delta(\mu,\pi)
&\le 2\xi\sum_{t=1}^{T}\mathbb E_{y_{\le t}\sim\mu}\left[
\left|\frac{\pi(y_t\mid s_t)}{\mu(y_t\mid s_t)}-1\right|
\right]\\
&=2\xi\sum_{t=1}^{T}\mathbb E_{s_t\sim\mu}\left[
\sum_{y_t}\mu(y_t\mid s_t)
\left|\frac{\pi(y_t\mid s_t)}{\mu(y_t\mid s_t)}-1\right|
\right]\\
&=2\xi\sum_{t=1}^{T}\mathbb E_{s_t\sim\mu}\Big[2D_{\mathrm{TV}}(\mu(\cdot\mid s_t)\|\pi(\cdot\mid s_t))\Big]\\
&=4\xi\,\mathbb E_{y\sim\mu}\left[\sum_{t=1}^{T}D_{\mathrm{TV}}\bigl(\mu(\cdot\mid s_t)\,\|\,\pi(\cdot\mid s_t)\bigr)\right].
\end{align*}
Substituting this into Equation~\eqref{eq:app-identity} proves Equation~\eqref{eq:app-avg-bound}.

The two forms can be combined into the immediate corollary:
\begin{align}
\mathcal J(\pi)-\mathcal J(\mu)
\ge L'_{\mu}(\pi)-\min\Bigg(
2\xi T(T-1)\bigl[D_{\mathrm{TV}}^{\max}(\mu,\pi)\bigr]^2,
4\xi\,\mathbb E_{y\sim\mu}\left[\sum_{t=1}^{T}D_{\mathrm{TV}}\bigl(\mu(\cdot\mid s_t)\,\|\,\pi(\cdot\mid s_t)\bigr)\right]
\Bigg).
\label{eq:app-composite}
\end{align}
The max-divergence form is sharper for very small updates, whereas the average-divergence form avoids the quadratic dependence on the horizon and is more useful for long LLM responses.

\section{How Does Each Stability Approach Work in Async RL?}
\label{app:stabilizers}

The main paper introduces \satr{} as one way to adapt the sampled trust-region surrogate to the observed mismatch distribution. The larger asynchronous-RL literature can be read through the same lens. Each method chooses one mathematical object built from the observed log-ratio or from closely related quantities, such as the granularity of the ratio, the denominator, a mask, a threshold, or a clip radius, and then modifies that object. Table~\ref{tab:designspace} summarizes the design space, and the subsections below expand the methods.

\begin{table}[t]
\centering
\caption{Async-RL stabilizers in one design space. The table records which mathematical object each method modifies and whether the gate depends on the update direction $\operatorname{sign}(\hat A_t(r_t-1))$.}
\label{tab:designspace}
\small
\setlength{\tabcolsep}{5pt}
\begin{adjustbox}{max width=\textwidth}
\begin{tabular}{@{}llllll@{}}
\toprule
Method & Unit & Discriminator & Constraint & Weight $w_{b,t}$ & Direction \\
\midrule
PPO~\citep{schulman2017ppo} & token & $|r_{b,t}-1|$ & $|r_{b,t}-1|\le\varepsilon$ & $\{0,1\}$ & asymmetric \\
TRPO~\citep{schulman2015trpo} & state & $\dtv[s_t]$ & $\dtv[s_t]\le\delta_{\mathrm{TR}}$ & --- & asymmetric \\
DPPO~\citep{qi2026dppo} & token & $D_t^{\mathrm{TV}}$ & $D_t^{\mathrm{TV}}\le\delta$ & $\{0,1\}$ & asymmetric \\
DRPO~\citep{yao2026drpo} & token & $D_t^{\mathrm{TV}}$ & soft quadratic penalty & $[1-\tfrac1\delta,1+\tfrac1\delta]$ & asymmetric \\
TIS~\citep{zheng2025tis} & token & $r_{b,t}$ & one-sided cap & $\min(r_{b,t},C)$ & one-sided \\
IcePop~\citep{zhao2025icepop,hou2026sao} & token & train/inference ratio & symmetric range gate & $\{0,1\}$ or in-range weight & symmetric \\
KPop~\citep{guo2026kpop} & token & Bernoulli KL & bidirectional threshold & $\{0,1\}$ & symmetric \\
GSPO~\citep{zheng2025gspo} & sequence & $|\rho_b^{\mathrm{seq}}-1|$ & sequence clip & $\{0,1\}$ & asymmetric \\
SeqClip~\citep{tencent2025seqclip} & sequence & $\rho_b^{\mathrm{seq}}$ & $\rho_b^{\mathrm{seq}}\in[\alpha_{\mathrm{seq}},\beta_{\mathrm{seq}}]$ & $\{0,1\}$ & symmetric \\
R3~\citep{ma2025r3} & forward pass & routing mask & replay rollout Top-$K$ mask & --- & --- \\
\rowcolor{metabg!8}
\satr{} (ours) & token & $|d_{b,t}|$ and $|r_{b,t}-1|$ & adaptive clip interval & $\{0,1\}$ & asymmetric, adaptive \\
\bottomrule
\end{tabular}
\end{adjustbox}
\end{table}

\begin{figure}[p]
\centering
\begin{subfigure}[t]{0.85\textwidth}
    \includegraphics[width=\linewidth]{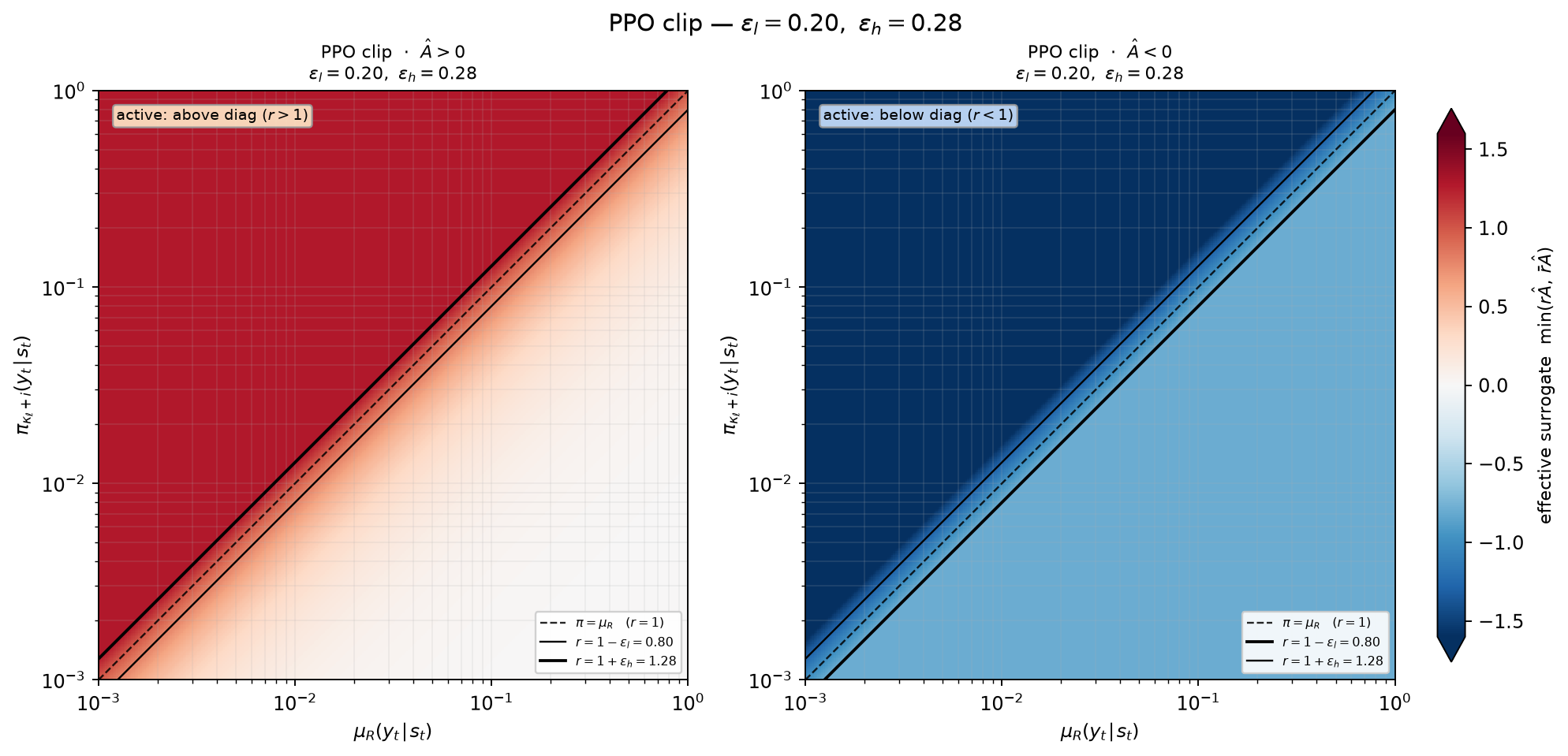}
    \caption{PPO}
\end{subfigure}
\caption{Asymmetric-clipping family, part I: PPO.}
\label{fig:geometry}
\end{figure}

\clearpage

\begin{figure}[p]
\centering
\begin{subfigure}[t]{0.85\textwidth}
    \includegraphics[width=\linewidth]{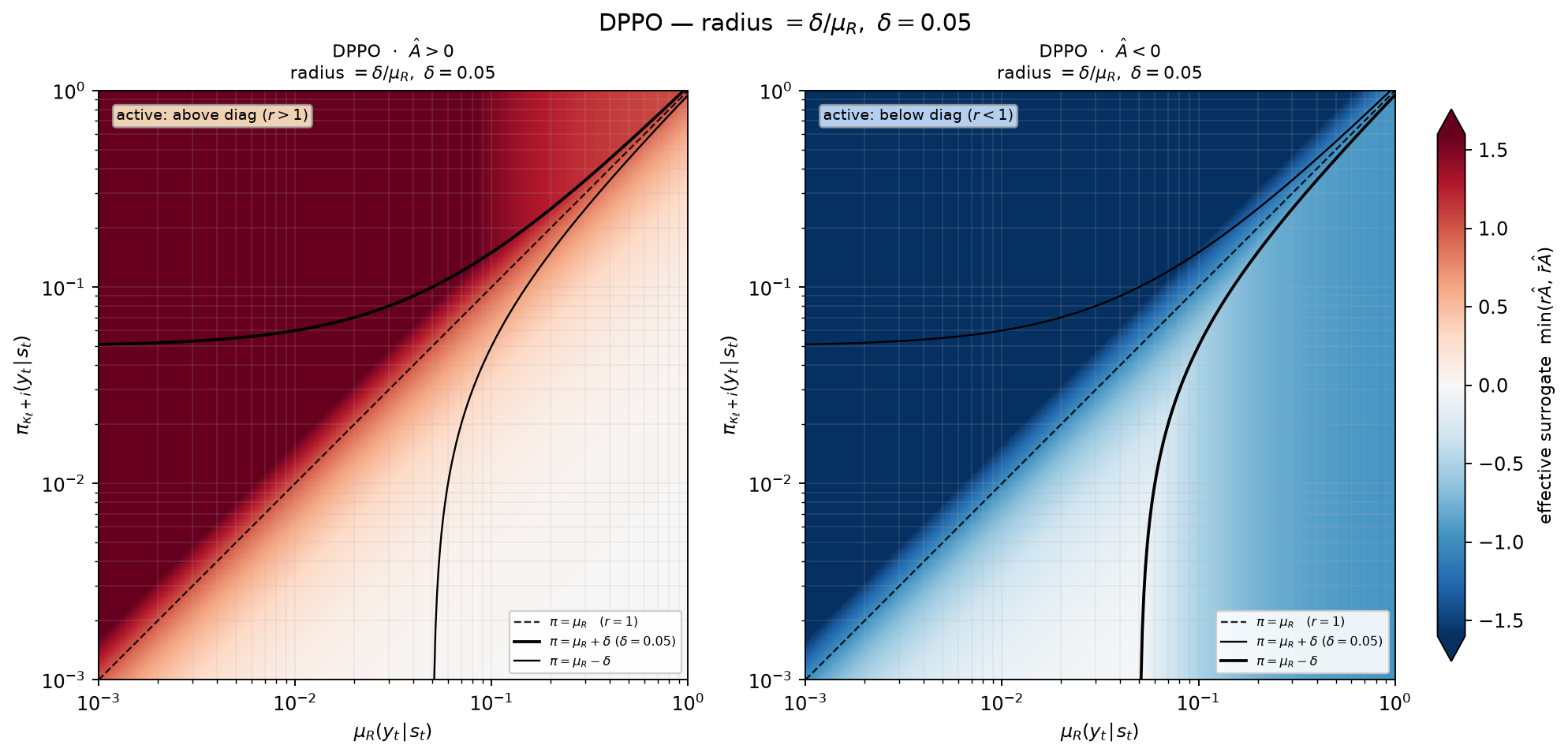}
    \caption{DPPO}
\end{subfigure}\\[0.6em]
\begin{subfigure}[t]{0.85\textwidth}
    \includegraphics[width=\linewidth]{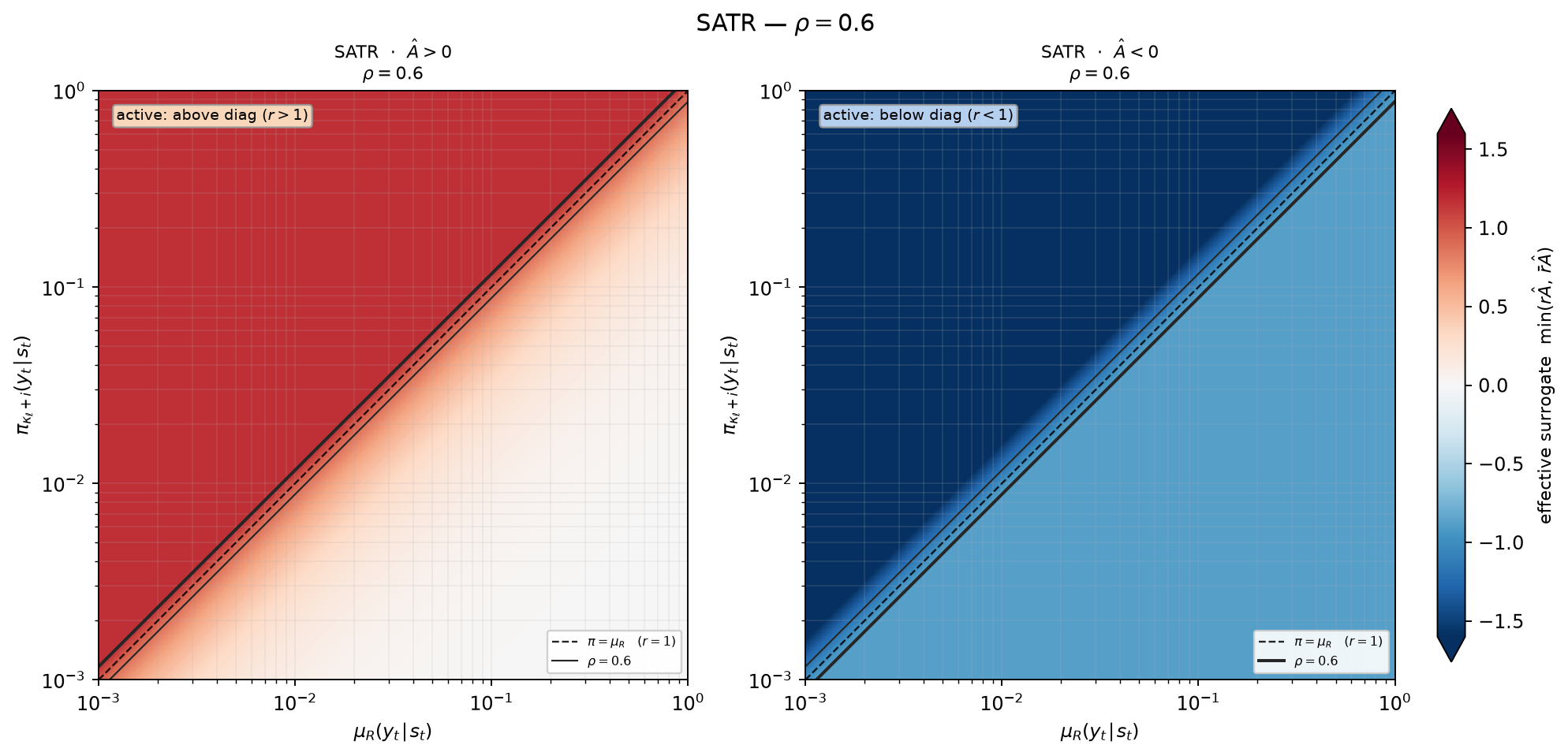}
    \caption{\satr{}}
\end{subfigure}
\caption{Asymmetric-clipping family, part II: DPPO and \satr{}. In this case, the contraction factor of \satr{} is set to $0.6$, which is staleness-adaptive in implementation.}
\label{fig:geometry-continued}
\end{figure}

\clearpage

\subsection{GSPO and SeqClip}
\label{app:gspo-seqclip}

GSPO moves importance weighting from the token to the response by defining the length-normalized sequence ratio:
\begin{equation}
\rho_b^{\mathrm{seq}}
=\left(\frac{\pi^{(j)}(y_b\mid x)}{\mu(y_b\mid x)}\right)^{1/T_b}
=\exp\left(\frac{1}{T_b}\sum_{t=1}^{T_b}d_{b,t}\right),
\label{eq:app-gspo}
\end{equation}
and clipping at the response level:
\begin{equation}
\mathcal S_{\mathrm{GSPO}}
=\mathbb E_{\{y_b\}\sim\mu}
\left[
\frac{1}{G}\sum_{b=1}^{G}
\min\Bigl(\rho_b^{\mathrm{seq}}\hat A_b,
\operatorname{clip}(\rho_b^{\mathrm{seq}},1-\varepsilon,1+\varepsilon)\hat A_b\Bigr)
\right].
\label{eq:app-gspo-loss}
\end{equation}
Length normalization removes the automatic exponential dependence on response length that the raw product of token ratios would induce, although token-level deviations can still reinforce or cancel. SeqClip adds a hard response-level gate:
\begin{equation}
\mathcal S_{\mathrm{SeqClip}}
=\mathbb E_{\{y_b\}\sim\mu}
\left[
\frac{1}{G}\sum_{b=1}^{G}
\mathbf 1\{\alpha_{\mathrm{seq}}\le \rho_b^{\mathrm{seq}}\le \beta_{\mathrm{seq}}\}
\min\Bigl(\rho_b^{\mathrm{seq}}\hat A_b,
\operatorname{clip}(\rho_b^{\mathrm{seq}},1-\varepsilon,1+\varepsilon)\hat A_b\Bigr)
\right].
\label{eq:app-seqclip}
\end{equation}
so responses outside a narrow interval are discarded altogether. Both methods smooth isolated token fluctuations by construction, but neither removes routing or engine mismatch, which is why the empirical combination with R3 remains relevant.

\paragraph{Empirical comparison.}
Equations~\eqref{eq:app-gspo} and~\eqref{eq:app-seqclip} define the sequence-level ratio and the hard response-level gate. Table~\ref{tab:app-exp-gspo} summarizes the matched GSPO comparison under the shared setup of Appendix~\ref{app:expdetails}. GSPO stays above GRPO at both configured lags, with $32.25$ against $31.25$ at lag $1$ and $31.46$ against $30.17$ at lag $8$. Both lag-$8$ runs later collapse in the logged window. The comparison therefore supports a difference in weighting granularity, but by itself it does not isolate which mismatch source drives the later instability. No standalone SeqClip run is reported in the current study.

\begin{table}[h]
\centering
\caption{Reported best AIME24 pass@1 for the GSPO comparison.}
\label{tab:app-exp-gspo}
\small
\begin{tabular}{@{}lcc@{}}
\toprule
Method & lag $1$ & lag $8$ \\
\midrule
GRPO & $31.25$ & $30.17^{\ddagger}$ \\
GSPO & $32.25$ & $31.46^{\ddagger}$ \\
\bottomrule
\end{tabular}
\end{table}

\subsection{Rollout Routing Replay}
\label{app:r3-detail}

R3 addresses mismatch from discrete MoE routing. Even at nominally identical weights, SGLang and Megatron can choose different Top-$K$ experts for the same token, and asynchronous updates add another source of variation. R3 replays the rollout-time Top-$K$ mask while keeping the current trainer logits in the softmax denominator,
\begin{equation}
g^{(j)}_{\mathrm{replay},e}
=\frac{I^{(\ell)}_{\mathrm{infer},e}\,\exp\bigl(z^{(j)}_{\mathrm{train},e}\bigr)}
{\sum_{e'}I^{(\ell)}_{\mathrm{infer},e'}\,\exp\bigl(z^{(j)}_{\mathrm{train},e'}\bigr)},
\qquad
\mathbf h^{(j)}_{\mathrm{replay}}
=\sum_{e=1}^{M}g^{(j)}_{\mathrm{replay},e}\,\mathcal E_e^{(j)}(\mathbf h_{\mathrm{train}}).
\label{eq:r3}
\end{equation}
This aligns the discrete expert set with rollout while still allowing gradients to flow through the current router. It does not reproduce the entire rollout forward pass, since hidden states, logits, expert weights, and downstream layers may still differ, but it removes one important mismatch source at the routing branch.

\paragraph{Empirical comparison.}
Equation~\eqref{eq:r3} defines the replayed routing gate. Table~\ref{tab:app-exp-r3} summarizes the matched R3 comparison. The strongest gain appears for the GSPO pairing, which reaches $34.00$ at lag $1$ and $33.13$ at lag $8$. In the reported summary, the lag-$8$ GSPO w/ R3 run is not marked as a later-collapse setting. GRPO w/ R3 improves over GRPO at lag $1$, but remains below GSPO w/ R3 at both lags. These numbers are consistent with routing replay reducing MoE routing mismatch on the reported stack, but they do not establish a unique causal path.

\begin{table}[h]
\centering
\caption{Reported best AIME24 pass@1 for the R3 comparison.}
\label{tab:app-exp-r3}
\small
\begin{tabular}{@{}lcc@{}}
\toprule
Method & lag $1$ & lag $8$ \\
\midrule
GRPO & $31.25$ & $30.17^{\ddagger}$ \\
GSPO & $32.25$ & $31.46^{\ddagger}$ \\
GRPO w/ R3 & $31.83$ & $29.58$ \\
GSPO w/ R3 & $34.00$ & $33.13$ \\
\bottomrule
\end{tabular}
\end{table}

\subsection{TIS, IcePop, and KPop}
\label{app:token-stabilizers}

TIS, IcePop, and KPop act on sampled tokens rather than whole responses. TIS is the mildest intervention: it caps only the overshoot side of the importance weight:
\begin{equation}
w_{b,t}^{\mathrm{TIS}}=\min\bigl(r_{b,t},C\bigr),
\qquad
\mathcal S_{\mathrm{TIS}}
=\mathbb E\left[\sum_{b,t}\operatorname{sg}[w_{b,t}^{\mathrm{TIS}}]\hat A_{b,t}\log\pi^{(j)}(y_{b,t}\mid s_{b,t})\right].
\label{eq:app-tis}
\end{equation}
It changes gradient magnitude but does not alter the clip interval. IcePop instead applies a two-sided gate to the train-engine versus inference-engine ratio:
\begin{equation}
k_{b,t}=\frac{\tilde\mu_{b,t}(y_{b,t}\mid s_{b,t})}{\mu_{b,t}(y_{b,t}\mid s_{b,t})},
\qquad
M_{b,t}^{\mathrm{IcePop}}=k_{b,t}\,\mathbf 1\{\alpha\le k_{b,t}\le\beta\},
\label{eq:app-icepop}
\end{equation}
which is often simplified to a binary in-range mask. KPop replaces raw ratios with a Bernoulli-KL test:
\begin{equation}
D_{\mathrm{KL}}^{\mathrm{Bern}}(p\|q)=p\log\frac{p}{q}+(1-p)\log\frac{1-p}{1-q},
\qquad
M_{b,t}^{\mathrm{KPop}}
=\mathbf 1\bigl\{D_{\mathrm{KL}}^{\mathrm{Bern}}(p_{b,t}^{\pi}\|p_{b,t}^{\mu})\le\phi\bigr\}
\mathbf 1\bigl\{D_{\mathrm{KL}}^{\mathrm{Bern}}(p_{b,t}^{\mu}\|p_{b,t}^{\pi})\le\phi\bigr\}.
\label{eq:app-kpop}
\end{equation}
where $p_{b,t}^{\pi}$ and $p_{b,t}^{\mu}$ are the trainer- and rollout-side probabilities of the sampled token. IcePop and KPop are symmetric in direction: they may discard tokens whose local gradient would move the sampled ratio back toward the behavior reference.

\paragraph{Empirical comparison.}
Equations~\eqref{eq:app-tis}--\eqref{eq:app-kpop} summarize the token-level mechanisms. The current study reports matched runs for TIS and its compositions with R3, but not standalone AIME24 grids for IcePop or KPop. Table~\ref{tab:app-exp-tis} shows that the TIS family remains in a relatively narrow band from $31.21$ to $32.92$ across the two configured lags, and none of these runs is marked as a later-collapse setting. The strongest lag-$8$ value in this family is $32.92$ for GSPO w/ TIS w/ R3.

\begin{table}[h]
\centering
\caption{Reported best AIME24 pass@1 for the TIS comparison.}
\label{tab:app-exp-tis}
\small
\begin{tabular}{@{}lcc@{}}
\toprule
Method & lag $1$ & lag $8$ \\
\midrule
GRPO w/ TIS & $31.62$ & $31.46$ \\
GSPO w/ TIS & $32.88$ & $31.21$ \\
GRPO w/ TIS w/ R3 & $31.79$ & $31.79$ \\
GSPO w/ TIS w/ R3 & $31.88$ & $32.92$ \\
\bottomrule
\end{tabular}
\end{table}

\subsection{Asymmetric clipping}
\label{app:asymclip}

The comparison most relevant for this paper is between PPO, DPPO, and \satr{}, since all three act directly on sampled-token updates through an asymmetric trust-region-style gate. For a fixed sampled token with ratio $r_t$ and advantage $\hat A_t$, the unclipped surrogate has derivative $\hat A_t$ with respect to $r_t$, so the sign of $\hat A_t$ determines the locally preferred direction. Relative to $r_t=1$, the sampled update is locally diverging when $\operatorname{sign}(\hat A_t(r_t-1))>0$ and locally converging when $\operatorname{sign}(\hat A_t(r_t-1))<0$. The three methods differ in how they decide when an outward move should be stopped: PPO uses a fixed ratio interval, DPPO replaces that fixed interval by a token-dependent threshold induced by a sampled $D_{\mathrm{TV}}$ bound, and \satr{} further makes the effective clip radius contract on the high-staleness tail while leaving pull-back updates open.

Away from the kinks, PPO therefore keeps the gate:
\begin{equation}
M_t^{\mathrm{PPO}}=\mathbf 1\Bigl\{\operatorname{sign}(\hat A_t(r_t-1))\le 0\ \vee\ |r_t-1|\le\varepsilon\Bigr\},
\label{eq:app-ppo-mask}
\end{equation}
which clips only outward moves beyond the nominal interval and always preserves pull-back updates.

DPPO keeps the same directional logic but changes the discriminator. Under the sampled $D_{\mathrm{TV}}$ discriminator:
\begin{equation}
D_t^{\mathrm{TV}}=\bigl|\pi^{(j)}(y_t\mid s_t)-\mu(y_t\mid s_t)\bigr|=\mu(y_t\mid s_t)\,|r_t-1|,
\label{eq:app-bintv}
\end{equation}
so $D_t^{\mathrm{TV}}\le\delta$ is equivalent to:
\begin{equation}
|r_t-1|\le \frac{\delta}{\mu(y_t\mid s_t)}.
\label{eq:app-dppo-eps}
\end{equation}
The induced mask is:
\begin{equation}
M_t^{\mathrm{DPPO}}
=\mathbf 1\Bigl\{\operatorname{sign}(\hat A_t(r_t-1))\le 0\ \vee\ |r_t-1|\le\delta/\mu(y_t\mid s_t)\Bigr\},
\label{eq:app-dppo-mask}
\end{equation}
which is loose on long-tail tokens and tight on head tokens. SPO and DRPO replace hard clipping with quadratic penalties. SPO uses:
\begin{equation}
\mathcal S_{\mathrm{SPO}}(\pi)
=\mathbb E\left[\sum_t\left(r_t\hat A_t-\frac{|\hat A_t|}{2\varepsilon}(r_t-1)^2\right)\right],
\label{eq:app-spo}
\end{equation}
whose stationary point lies at $r_t^{\star}=1+\operatorname{sign}(\hat A_t)\varepsilon$, exactly on PPO's boundary. DRPO makes the same replacement in DPPO's sampled $D_{\mathrm{TV}}$ geometry:
\begin{equation}
\mathcal S_{\mathrm{DRPO}}(\pi^{(j)})
=\mathbb E_{y\sim\mu}\left[\sum_t\left(r_t\hat A_t-\frac{|\hat A_t|}{2\delta}\,\mu(y_t\mid s_t)(r_t-1)^2\right)\right].
\label{eq:app-drpo}
\end{equation}
Unlike \satr{}, however, their thresholds are fixed in advance rather than inferred from the current mismatch tail.

\paragraph{Empirical comparison.}
Equations~\eqref{eq:app-dppo-mask} and~\eqref{eq:app-drpo} place DPPO and DRPO inside the asymmetric-clipping family. Table~\ref{tab:app-exp-dppo} summarizes the reported DPPO comparison. DPPO reaches $33.96$ at lag $1$ and $32.71$ at lag $8$. It stays above the GRPO baseline at both lags and outside the runs marked as later collapse. This is consistent with a sampled $D_{\mathrm{TV}}$ gate improving stability on the reported stack, but it does not explain the remaining gap to \satr-GSPO w/ R3 or even to \satr-GRPO w/ R3.

\begin{table}[h]
\centering
\caption{Reported best AIME24 pass@1 for the DPPO comparison.}
\label{tab:app-exp-dppo}
\small
\begin{tabular}{@{}lcc@{}}
\toprule
Method & lag $1$ & lag $8$ \\
\midrule
GRPO & $31.25$ & $30.17^{\ddagger}$ \\
GSPO & $32.25$ & $31.46^{\ddagger}$ \\
DPPO & $33.96$ & $32.71$ \\
\bottomrule
\end{tabular}
\end{table}

\subsection{Top-$p$ configuration}
\label{app:topp}

Support mismatch appears when rollout truncates the behavior distribution but training evaluates an untruncated target policy. If rollout uses top-$p$ below $1$, actions outside the nucleus have zero behavior probability, so absolute continuity $\pi^{(j)}\ll\mu_{b,t}$ can fail. One fix is to replay the rollout truncation mask inside the trainer softmax, as in MAI-Thinking-1~\citep{microsoft2025mai}. If $\mathcal T_{b,t}\subseteq\mathcal V$ denotes the rollout nucleus, training uses:
\begin{equation}
\tilde{\mathrm{logit}}_{b,t}(a)=
\begin{cases}
\mathrm{logit}_{b,t}(a), & a\in\mathcal T_{b,t},\\
-\infty, & a\notin\mathcal T_{b,t},
\end{cases}
\qquad
\pi^{(j)}(a\mid s_{b,t})=\operatorname{softmax}(\tilde{\mathrm{logit}}_{b,t})(a).
\label{eq:app-top-p-replay}
\end{equation}
Our experiments use the simpler alternative of disabling truncation at rollout entirely by setting temperature $1.0$, top-$p$ to $1.0$, and top-$k$ to $-1$, so the rollout support matches the full vocabulary.

\paragraph{Empirical note.}
Equation~\eqref{eq:app-top-p-replay} gives the replayed support mask when rollout truncation is enabled. The companion analysis does not report a separate AIME24 grid for top-$p$. In the present experiments we instead disable truncation at rollout with temperature $1.0$, top-$p$ set to $1.0$, and top-$k$ set to $-1$, so rollout support matches the full vocabulary. Under this configuration, the reported trainer-to-rollout KL remains on the order of $10^{-3}$, which is comparable to the post-R3 regime on the same stack.

\subsection{Using rollout log probabilities}
\label{app:rollout-lp}

A separate choice concerns the denominator of the sampled ratio. Standard asynchronous PPO/GRPO often uses the trainer-side recomputation $\tilde\mu_{b,t}$ as the reference log-probability. Replacing it with the rollout log-probability restores the actual behavior denominator:
\begin{equation}
\log\tilde\mu_{b,t}(y_{b,t}\mid s_{b,t})\leftarrow \log\mu_{b,t}(y_{b,t}\mid s_{b,t}),
\qquad
r_{b,t}^{\mathrm{rollout-lp}}
=\exp\Bigl(\log\pi^{(j)}(y_{b,t}\mid s_{b,t})-\log\mu_{b,t}(y_{b,t}\mid s_{b,t})\Bigr)=r_{b,t}.
\label{eq:app-rollout-lp}
\end{equation}
Under the support condition, this is the denominator required by the importance-sampling identity. It may also increase the observed variance of sampled ratios, because implementation mismatch is no longer absorbed by trainer-side recomputation. The change therefore exposes mismatch more faithfully, but does not reduce policy-version lag by itself.

\paragraph{Empirical comparison.}
Equation~\eqref{eq:app-rollout-lp} defines the denominator switch. Table~\ref{tab:app-exp-rolloutlp} records the observed variance of the sampled log-ratio with and without rollout log probabilities. The switch exposes more implementation mismatch to the optimizer and increases the measured variance in the reported runs. This is a property of the estimator used in optimization rather than a change to the deterministic population bound itself.

\begin{table}[h]
\centering
\caption{Observed variance of the sampled log-ratio when rollout log probabilities are used as the denominator.}
\label{tab:app-exp-rolloutlp}
\small
\begin{tabular}{@{}lcc@{}}
\toprule
Method & lag $1$ & lag $8$ \\
\midrule
GRPO & $2.6\times 10^{-4}$ & $1.4\times 10^{-4}$ \\
GRPO w/ rollout log probabilities & $3.2\times 10^{-4}$ & $2.7\times 10^{-4}$ \\
\bottomrule
\end{tabular}
\end{table}

\section{Experimental Details}
\label{app:expdetails}

The backbone is Qwen3-30B-A3B-Base, an MoE model with $48$ layers, hidden size $2048$, $128$ experts with Top-$8$ routing, MoE-FFN width $768$, a softmax router, no auxiliary routing loss, and RoPE base $10^6$. The benchmark is AIME24 with $30$ problems. Each evaluation uses $16$ samples per problem, a $30$k response cap, top-$p$ set to $1$, and is run every five training iterations. Each logged evaluation score averages four evaluation runs, and the base-model reference is $9.38$.

All experiments run on the slime asynchronous RL framework with SGLang as the rollout engine and Megatron as the trainer. The configured lag is controlled at $n\in\{1,8\}$, with a weight-broadcast interval of $8$ for the lag-$8$ runs. All runs in a comparison share the optimizer, data, and reward. Table~\ref{tab:main-combined} reports both the best logged AIME24 checkpoint over the training window and the corresponding last-epoch sampled mismatch values. The symbol $\ddagger$ marks runs that later collapse and do not recover in the logged window, even though the best-so-far checkpoint is still reported.

Training data consists of the union of DAPO-Math-17k~\citep{yu2025dapoopensourcellmreinforcement} and Dolci-RL-Zero-Math-7B~\citep{olmo2026olmo3}, with $30{,}712$ prompts in total. We use a chat template, rollout shuffling, and balanced sampling. The reward is the rule-based PrimeMath verifier with multi-pattern answer extraction, Hendrycks-MATH normalization, and SymPy equivalence checking. Optimization uses Adam with $\beta_1=0.9$, $\beta_2=0.98$, and constant learning rate $10^{-6}$, without warmup or decay. Each iteration consumes $256$ prompts and $16$ samples per prompt, for $4096$ responses in total, with one minor step per broadcast and $544$ rollout iterations. Rollout and evaluation use a maximum response length of $32$k tokens and dynamic batching with a $32$k per-GPU token cap. Sampling is performed at temperature $1.0$, top-$p$ set to $1.0$, and top-$k$ set to $-1$, so the rollout distribution has full-vocabulary support. The base clip radii are $(\elow,\ehigh)=(0.2,0.2)$. The KL regularizer is low-variance KL with coefficient $10^{-3}$, and the entropy bonus is zero. For \satr{}, we use the Hill kernel with exponent $2$ and reference quantile $0.90$. Hardware consists of one node with eight GPUs, using GPU0--3 for training and GPU4--7 for rollout, with TP$=4$, PP$=1$, CP$=1$, EP$=4$, ETP$=1$, sequence parallelism, bf16 arithmetic with fp32 gradient all-reduce and fp32 attention softmax, the FlashAttention backend, and SGLang memory fraction set to $0.85$.

The per-method settings differ only in the ratio and KL granularity and in the clipping or gating rule. GRPO uses per-token ratios and per-token KL with the standard asymmetric PPO clip. GSPO keeps the same outer objective but replaces the token ratio and KL by a sequence-level ratio and a sequence-level mean KL that is all-gathered and broadcast to all active tokens of the response. DPPO multiplies the GRPO clip by a hard mask that zeros sampled tokens only when the local update is outward and the sampled $D_{\mathrm{TV}}$ discriminator exceeds its threshold; in the reported runs, this is the DPPO-$D_{\mathrm{TV}}$ variant and the threshold is $10^{-3}$. When \satr{} is layered on top of any base recipe, it contracts the shared radii using the ratio scalar of that recipe, namely the token ratio for GRPO and the broadcast sequence ratio for GSPO, and it is bit-identical to the base loss when disabled. R3 and TIS are toggled independently and therefore compose with both fixed-clip and adaptive-clip objectives.

\paragraph{Metrics and statistical scope}
\label{app:metrics-scope}
We report (i) the best-observed AIME24 pass@1 over the logged training window; AIME24 has $30$ problems, each evaluation uses $16$ samples per problem, runs every five training iterations, and each logged score averages four evaluation runs (base-model reference: $9.38$); and (ii) the sampled training--inference mismatch:
\begin{equation}
\dpi\;=\;\E_{(b,t)\in\mathcal T_{\mathrm{act}}}\big|\log r^{\mathrm{tok}}_{b,t}\big|,
\label{eq:mismatch-metric}
\end{equation}
computed from the token ratio of Equation~\eqref{eq:gspo-ratio} for all methods including GSPO, reported in the lower panel of Table~\ref{tab:main-combined} and as paired per-step curves in Figure~\ref{fig:drift-pair}. Every configuration is a single training seed, so each cell is one run: the numbers are a snapshot rather than confidence-interval estimates, and where the text calls two $\dpi$ values ``close'' we mean within the trailing-25\% within-run fluctuation, not a multi-seed test. Lower $\dpi$ means less observed sampled mismatch; it is not a full-$D_{\mathrm{TV}}$ measurement and not by itself a return ranking.

\begin{table}[h]
\centering
\caption{Per-method reproducible hyper-parameters. All other settings are shared and summarized above.}
\label{tab:hparams}
\small
\setlength{\tabcolsep}{7pt}
\begin{tabular}{@{}lccc@{}}
\toprule
Hyper-parameter & GRPO & GSPO & DPPO \\
\midrule
advantage estimator & GRPO & GSPO & DPPO \\
ratio / KL granularity & per-token & per-sequence & per-token \\
$(\elow,\ehigh)$ & $(0.2,0.2)$ & $(0.2,0.2)$ & $(0.2,0.2)$ \\
KL-loss coefficient & $10^{-3}$ & $10^{-3}$ & $10^{-3}$ \\
KL-loss type & low-variance KL & low-variance KL & low-variance KL \\
entropy coefficient & $0$ & $0$ & $0$ \\
extra mechanism & $-$ & sequence-level KL all-gather & $D_{\mathrm{TV}}$ hard mask $\delta=10^{-3}$ \\
needs rollout log-prob & no & no & yes \\
\bottomrule
\end{tabular}
\end{table}

Each cell of Table~\ref{tab:main-combined} is a single run at a fixed seed, so the empirical scope is intentionally narrow. The $\dpi$ curves in Figure~\ref{fig:drift-pair} are raw per-step logs at a four-step stride without smoothing. Steady-state values quoted in the caption average the trailing $25\%$ of each run, whereas the lower $\dpi$ block of Table~\ref{tab:main-combined} reports last-epoch means, so the two summaries can differ slightly for the same configuration. At lag $1$, the trailing-$25\%$ values are approximately $0.0102$ for GRPO, $0.0103$ for DPPO, $0.0101$ for GSPO w/ TIS, $0.0104$ for \satr-GSPO, $0.0061$ for GSPO w/ TIS w/ R3, $0.0055$ for GRPO w/ R3, $0.0057$ for GSPO w/ R3, and $0.0058$ for \satr-GSPO w/ R3. At lag $8$, the corresponding values are approximately $0.0068$ for GSPO w/ TIS w/ R3, $0.0073$ for GSPO w/ R3, $0.0085$ for \satr-GRPO w/ R3, $0.0097$ for GRPO w/ R3, and $0.0108$ for \satr-GSPO, while the remaining shown runs lie roughly in the $0.0103$--$0.0131$ range and the GSPO baseline diverges beyond the plotted axis.

\FloatBarrier

\section{\satr{} Runtime Diagnostics}
\label{app:monitor}

\begin{table}[h]
\centering
\caption{Diagnostics emitted at every optimizer step by the adaptive rule in Equations~\eqref{eq:gate}--\eqref{eq:adaptive-eps}. They describe the implemented gate on sampled tokens rather than full-vocabulary $D_{\mathrm{TV}}$.}
\label{tab:monitor}
\small
\begin{tabular}{@{}ll@{}}
\toprule
Quantity & Meaning \\
\midrule
$\E_{(b,t)\in\mathcal T_{\mathrm{act}}}[\elow\,c_{-,b,t}]$ & Mean effective lower radius on active sampled tokens. \\
$\E_{(b,t)\in\mathcal T_{\mathrm{act}}}[\ehigh\,c_{+,b,t}]$ & Mean effective upper radius on active sampled tokens. \\
$\E_{(b,t)\in\mathcal T_{\mathrm{act}}}\big[\mathbf 1\{q>0,\ |d_{b,t}|>q\}\big]$ & Realized gate rate; ties can make it smaller than $10\%$. \\
$\min_{(b,t)\in\mathcal T_{\mathrm{act}}}c_{\pm,b,t}$ & Strongest contraction applied to a sampled token in the microbatch. \\
$q$ & Detached microbatch quantile reference from Equation~\eqref{eq:quantile}. \\
\bottomrule
\end{tabular}
\end{table}

\end{document}